\documentclass[sts]{imsart}

\RequirePackage{amsthm,amsmath,amsfonts,amssymb}
\RequirePackage[numbers]{natbib}
\RequirePackage[colorlinks,citecolor=blue,urlcolor=blue]{hyperref}
\usepackage[createShortEnv,conf={end,restate}]{proof-at-the-end}
\RequirePackage[pdftex]{graphicx}

\numberwithin{equation}{section}
\theoremstyle{plain}

\newtheorem{theorem}{Theorem}[section]
\newtheorem{corollary}[theorem]{Corollary}

\theoremstyle{remark}
\newtheorem{definition}[theorem]{Definition}

\usepackage{math-macros}
\usepackage{microtype}
\usepackage{algorithm}
\usepackage{algpseudocode}
\usepackage{wrapfig}
\usepackage{tabularx}
\usepackage{booktabs} 

\begin{document}

\begin{frontmatter}
	\title{A Unified Theory of Exact Inference\\ and Learning in Exponential Family\\ Latent Variable Models}
	\runtitle{Conjugated Harmoniums}

	\begin{aug}
		\author[A]{\fnms{Sacha}~\snm{Sokoloski}\ead[label=e1]{sacha.sokoloski@mailbox.org}}
		\address[A]{Hertie Institute for AI in Brain Health, University of Tübingen\printead[presep={\ }]{e1}.}
	\end{aug}

	\begin{abstract}
		Bayes' rule describes how to infer posterior beliefs about latent variables given observations, and inference is a critical step in learning algorithms for latent variable models (LVMs). Although there are exact algorithms for inference and learning for certain LVMs such as linear Gaussian models and mixture models, researchers must typically develop approximate inference and learning algorithms when applying novel LVMs. Here we study the line that separates LVMs that rely on approximation schemes from those that do not, and develop a general theory of exponential family LVMs for which inference and learning may be implemented exactly. Firstly, under mild assumptions about the exponential family form of the LVM, we derive a necessary and sufficient constraint on the parameters of the LVM under which the prior and posterior over the latent variables are in the same exponential family. We then show that a variety of well-known and novel models indeed have this constrained, exponential family form. Finally, we derive generalized inference and learning algorithms for these LVMs, and demonstrate them with a variety of examples. Our unified perspective facilitates both understanding and implementing exact inference and learning algorithms for a wide variety of models, and may guide researchers in the discovery of new models that avoid unnecessary approximations.
	\end{abstract}


\end{frontmatter}

\section{Introduction}

Latent variable models (LVMs) describe data in terms of observable random variables that we measure directly, and latent random variables that we do not measure, but that help explain our observations. The probabilistic formulation of LVMs affords general strategies for two fundamental operations when applying them: (i) Bayes' rule describes how to infer posterior beliefs about latent variables given prior beliefs and observations; and (ii) the principle of maximum likelihood --- or equivalently minimal cross-entropy, reduces learning to minimizing the cross-entropy of the LVM parameters given the data~\cite{roweis_unifying_1999,bishop_pattern_2006}. Because LVMs can often be scaled to achieve arbitrary model complexity, the central limitation of LVMs is rarely whether they are sufficiently powerful to model a given dataset, but rather whether inference and learning for a given LVM are computationally tractable.

There are numerous techniques for approximate inference and learning with LVMs~\cite{wainwright_graphical_2008,kingma_introduction_2019, martin_approximating_2024}. Variational methods, for example, approximate intractable LVM posteriors with tractable distributions by minimizing the divergence of the latter from the former, and these approximate posteriors can then be used to drive approximate learning algorithms for LVMs. The variational approach is particularly widespread in machine learning, where it has been used to develop powerful models of hierarchical~\cite{hinton_fast_2006,salakhutdinov_efficient_2012,vertes_flexible_2018} and dynamic latent structure~\cite{taylor_two_2011,boulanger-lewandowski_modeling_2012,durstewitz_state_2017}, which have in turn been applied to modelling biological neural circuits~\cite{beck_complex_2012}, cognitive modelling~\cite{salakhutdinov_learning_2013,lake_humanlevel_2015}, and image synthesis~\cite{ho_denoising_2020}.

Nevertheless, there are well-known LVMs for which inference and learning can be implemented exactly --- by exact inference and learning we mean, more or less, that there is a tractable expression for the posterior over the latent variables, and that we may minimize the cross-entropy of the model parameters directly, rather than via an approximation or lower-bound that might introduce approximation errors~\cite{shekhovtsov_vae_2021}. For example, there are closed-form expressions for the posteriors of mixture models and linear Gaussian models (which includes principal component analysis and factor analysis as special cases), and exact implementations of expectation-maximization (EM) for training them~\cite{bishop_pattern_2006}.

In this paper we study the boundary that separates models that require approximation techniques for inference and learning from those that do not, and derive a general theory of exact learning and inference for a broad class of exponential family LVMs. We begin our approach by analyzing the conditions under which priors and posteriors over latent variables share the same exponential family form. In general, priors are known as conjugate priors when they have the same parametric form as the posterior~\cite{diaconis_conjugate_1979,arnold_conjugate_1993}, and conjugate priors are widely-applied to inferring posteriors over the parameters of exponential family models. In this study we generalize this notion of conjugacy to LVMs, where there is a separation between latent variables and model parameters, and we use the machinery of conjugate priors to infer posteriors over the latent variables.

We show that if the likelihood of the observations given the latent variables is exponential family distributed, then the latent variable prior is conjugate to the likelihood if and only if (i) the LVM has a particular exponential-family form, and (ii) the LVM parameters satisfy a constraint. The exponential-family form in question has a long history~\cite{besag_spatial_1974,arnold_compatible_1989,arnold_conditionally_2001,yang_graphical_2015,tansey_vector-space_2015}, and models with this form have been variously referred to as conditionally specified distributions~\cite{arnold_compatible_1989}, vector space Markov random fields~\cite{tansey_vector-space_2015}, and exponential family harmoniums~\cite{smolensky_information_1986,welling_exponential_2005} --- here we refer to them simply as harmoniums. We then show that for harmoniums, the prior is conjugate to the likelihood if and only if it satisfies an equation on its exponential family parameters. We refer to harmoniums that satisfy this equation as conjugated harmoniums.

We show that both mixture models and linear Gaussian models are forms of conjugated harmonium. Outside of these well-known cases, we also study harmoniums defined variously in terms of von Mises distributions, Poisson distributions, and Boltzmann machines, and show that they are conjugated under certain conditions.

Inference and learning algorithms for LVMs are often developed on a case-by-case basis, yet we show that many algorithms are special cases of a set of general algorithms for inference and learning with conjugated harmoniums. This not only facilitates theoretical unification, but also simplifies the implementation of learning and inference algorithms, as it allows programmers to implement one set of algorithms for a wide variety of cases. Indeed, we have developed a collection of Haskell libraries for numerical optimization based on our theory, and all the models that we demonstrate in this paper were implemented with these libraries (code available as Supplementary Material~\cite{supplement} and at \url{https://github.com/alex404/goal}).

\section{Background}

In this section we introduce the exponential family background for our unified theory --- for a thorough treatment of exponential families see~\cite{amari_methods_2007,wainwright_graphical_2008}. We rely on several notational conventions, and although we introduce this notation over the course of the text, we summarize key features here.

In general, we use Latin letters (e.g. $x$ or $z$) for observations and latent states, and Greek letters (e.g. $\theta$ or $\eta$) for model parameters. We use lowercase letters (e.g. $x$ or $\theta$) to denote scalars, bold, lowercase letters (e.g. $\V x$ or $\eprms$) to denote vectors, and bold, capital letters (e.g.\ $\V X$ or $\iprms$) to denote matrices --- when the form of the variable (e.g.\ scalar or vector) is not specified, we default to non-bold, lowercase letters. We use capital, italic letters (e.g. $X$ and $Z$) to denote random variables, regardless if they are scalars, vectors, or matrices. We use calligraphic letters in the latter part of the Latin alphabet (e.g. $\mathcal X$ or $\mathcal Z$) to denote the sample space of a random variable, calligraphic letters otherwise (e.g.\ $\mathcal M$) to denote statistical models, and capital Greek letters (e.g. $\Theta$ or $\Eta$) to denote parameter spaces. We use $P$ and $p$ to denote probability distributions and densities of random variables, respectively, and $Q$ and $q$ to denote model distributions and densities, respectively. Finally, to indicate that a mathematical object is related to a particular random variable or sample space, we subscript it with capital, italic letters. For example, $P_X$ and $p_X$ denote the probability distribution and density of $X$, respectively. Similarly, $\mathcal M_X$ is a model on the sample space $\mathcal X$ of some random variable $X$ of interest, parameterized e.g.\ by the space $\Theta_X$, with contains elements $\eprms_X \in \Theta_X$.

\subsection{Exponential Families}

Consider a random variable $X$ on the sample space $\mathcal X$ with unknown distribution $P_X$, and suppose $X^{(1)}, \dots, X^{(n)}$ is an independent and identically distributed sample from $P_X$. One strategy for modelling $P_X$ based on the sample is to first define a ``sufficient'' statistic $\V s_X \colon \mathcal X \to \Eta_X$ that captures features of interest about $X$. We then look for a probability distribution $Q_X$ whose expectation of the sufficient statistic $\E_Q[\V s_X(X)]$ matches the average of the sufficient statistic over the data $\frac{1}{n}\sum_{i=1}^n \V s_X(X^{(i)})$, where the expected value of $f(X)$ under $Q_X$ is defined by $\E_Q[f(X)] = \int_{\mathcal X} f dQ_X$. To further constrain the space of possible distributions we also assume that $Q_X$ maximizes the entropy $E_Q[-\log q_X]$, where $q_X = dQ_X/d\mu_X$ is the density function (Radon-Nikodym derivative) of $Q_X$ with respect to some base measure $\mu_X$. Based on these assumptions, it can be shown that $q_X$ must have the exponential family form
\begin{equation}\label{eq:exponential-family-density}
	\log q_X(x) = \V s_X(x) \cdot \eprms_X - \psi_X(\eprms_X),
\end{equation}
where $\eprms_X$ are the natural parameters, and $\psi_X(\eprms_X) = \log \int_{\mathcal X}e^{\V s_X(x) \cdot \eprms_X}d\mu_X(x)$ is the log-partition function. We also refer to $\mprms_X = \E_Q[\V s_X(X)]$ as the mean parameters of $q_X$.

It is straightforward to show that $\partial_{\eprms_X} \psi_X(\eprms_X) = \E_Q[\V s_X(X)] = \mprms_X$, so that we may easily identify the mean parameters $\mprms_X$ of $q_X$ given its natural parameters $\eprms_X$. However, the natural parameters $\eprms_X$ for a given density $q_X$ might not be unique, in the sense that different natural parameters might yield the same density function as defined by Equation~\ref{eq:exponential-family-density}. To address this we may further assume that the sufficient statistic $\V s_X$ is minimal, in the sense that the component functions ${\{ s_{X,i}\}}_{i=1}^{d_X}$ are non-constant and linearly independent, where $d_X$ is the dimension of $\Eta_X$. If the sufficient statistic $\V s_X$ is minimal, then each $\eprms_X$ determines a unique density $q_X$, and $\partial_{\eprms_X} \psi_X$ is invertible.

A $d_X$-dimensional exponential family $\mathcal M_X$ is thus a manifold of probability densities defined by a sufficient statistic $\V s_X$ and a base measure $\mu_X$. The densities in $\mathcal M_X$ can be identified by their mean parameters $\mprms_X$ or natural parameters $\eprms_X$, and we denote the space of all mean and natural parameters by $\Eta_X$ and $\Theta_X$ respectively. A minimal exponential family $\mathcal M_X$ is an exponential family with a minimal sufficient statistic $\V s_X$, and the parameter spaces $\Theta_X$ and $\Eta_X$ of a minimal exponential family are isomorphic. In this case, the transition functions between the two parameter spaces, known as the forward mapping $\V \tau_X \colon \Theta_X \to \Eta_X$ and the backward mapping $\V \tau_X\inv \colon \Eta_X \to \Theta_X$, are given by the gradient $\V \tau_X(\eprms_X) = \partial_{\eprms_X}\psi_X(\eprms_X)$ and its inverse, respectively.

\subsection{Exponential Family Harmoniums}\label{sec:harmoniums}

An exponential family harmonium is a kind of product exponential family which includes various LVMs as special cases~\cite{smolensky_information_1986,welling_exponential_2005}. Given two exponential families $\mathcal M_X$ and $\mathcal M_Z$ that model the distributions of $X$ and $Z$, respectively, a harmonium $\mathcal H_{XZ}$ is an exponential family that models the joint distribution of $X$ and $Z$. Where $\mathcal M_X$ and $\mathcal M_Z$ have base measures $\mu_X$ and $\mu_Z$, and sufficient statistics $\V s_X$ and $\V s_Z$, respectively, the base measure of $\mathcal H_{XZ}$ is $\mu_{XZ} = \mu_X \cdot \mu_Z$, and its sufficient statistic is given by $\V s_{XZ}(x, z) = (\V s_X(x), \V s_Z(z), \V s_X(x) \otimes \V s_Z(z))$, where $\otimes$ is the outer product operator. In words, $\V s_{XZ}$ is the concatenation of all the component functions of $\V s_X$, $\V s_Z$, and $\V s_X \otimes \V s_Z$. We also note that $\mathcal H_{XZ}$ is minimal if both $\mathcal M_X$ and $\mathcal M_Z$ are minimal. More intuitively, $\mathcal H_{XZ}$ is the exponential family that comprises all densities $q_{XZ}$ with the form
\begin{multline}\label{eq:harmonium-density}
	\log q_{XZ}(x,z) = \V s_X(x) \cdot \eprms_X + \V s_Z(z) \cdot \eprms_Z + \\ \V s_X(x) \cdot \iprms_{XZ} \cdot \V s_Z(z) - \psi_{XZ}(\eprms_X, \eprms_Z, \iprms_{XZ}),
\end{multline}
where $\eprms_{XZ} = (\eprms_X, \eprms_Z, \iprms_{XZ})$ are the natural parameters of $q_{XZ}$, and $\psi_{XZ}$ is the log-partition function of $\mathcal H_{XZ}$. We refer to the parameters $\eprms_X$ and $\eprms_Z$ as biases, and $\iprms_{XZ}$ as interactions.

Although harmoniums can certainly model the joint distribution of two observable variables, we focus on the setting where one of the variables is latent, and so it will prove helpful to begin developing the language of LVMs. From here on out we will assume that the variables $X$ and $Z$ denote observable and latent random variables, respectively, and we refer to e.g.\ $\V s_X$ and $\V s_Z$ as the observable and latent sufficient statistics, and $\eprms_X$ and $\eprms_Z$ as the observable and latent biases, respectively.

Harmonium densities have a simple log-linear structure, but their marginal densities do not. In particular, for $q_{XZ} \in \mathcal H_{XZ}$ with natural parameters $\eprms_X$, $\eprms_Z$, and $\iprms_{XZ}$, it is easy to show that the observable density is given by
\begin{multline}\label{eq:harmonium-data-distribution}
	\log q_X(x) =  \V s_X(x) \cdot \eprms_X + \\ \psi_Z(\eprms_Z + \V s_X(x) \cdot \iprms_{XZ}) - \psi_{XZ}(\eprms_X, \eprms_Z, \iprms_{XZ}),
\end{multline}
and similarly the prior is given by
\begin{multline}\label{eq:harmonium-prior}
	\log q_Z(z) = \V s_Z(z) \cdot \eprms_Z + \\ \psi_X(\eprms_X + \iprms_{XZ} \cdot \V s_Z(z)) - \psi_{XZ}(\eprms_X, \eprms_Z, \iprms_{XZ}).
\end{multline}

In contrast, the conditional densities of a harmonium do inherit a linear structure, and have the form of generalized linear models~\cite{bishop_pattern_2006,yang_graphical_2012}. In particular, where the likelihood and posterior are defined by $q_{X \mid Z} = \frac{q_{XZ}}{q_Z}$ and $q_{Z \mid X} = \frac{q_{XZ}}{q_X}$, respectively, we may combine Equations~\ref{eq:harmonium-density},~\ref{eq:harmonium-data-distribution}, and~\ref{eq:harmonium-prior}, to conclude that
\begin{multline}\label{eq:likelihood-density}
	\log q_{X \mid Z}(x \mid z) = \\ \V s_X(x) \cdot (\eprms_X + \iprms_{XZ} \cdot \V s_Z(z)) - \psi_X(\eprms_X + \iprms_{XZ} \cdot \V s_Z(z)),
\end{multline}
and similarly that
\begin{multline}\label{eq:posterior-density}
	\log q_{Z \mid X}(z \mid x) = \\ \V s_Z(z) \cdot (\eprms_Z + \V s_X(x) \cdot \iprms_{XZ}) - \psi_Z(\eprms_Z + \V s_X(x) \cdot \iprms_{XZ}).
\end{multline}
When discussing conditional densities we denote e.g.\ the likelihood $q_{X \mid Z}$ at $z$ by $q_{X \mid Z = z}$, so that $q_{X \mid Z = z} \in \mathcal M_X$ is the exponential family density with natural parameters $\eprms_{X \mid Z}(z) = \eprms_X + \iprms_{XZ} \cdot \V s_Z(z)$. We also write e.g.\ $q_{X \mid Z} \in \mathcal M_X$ as shorthand for $q_{X \mid Z = z} \in \mathcal M_X, \forall z \in \mathcal Z$, to express that $q_{X \mid Z}$ is always a member of $\mathcal M_X$ for any $z$.

We defined harmoniums constructively as a product of component exponential families, but we may also define them intrinsically, as the most general families of densities with likelihoods and posteriors in pre-specified exponential families.

\begin{theorem}\label{thm:harmonium-specification}
	Suppose that $\mathcal H_{XZ}$ is the harmonium defined by the minimal exponential families $\mathcal M_X$ and $\mathcal M_Z$, and that $q_{XZ}$ is an arbitrary joint density over the sample space $\mathcal X \times \mathcal Z$ with respect to the product measure $\mu_{XZ} = \mu_X \cdot \mu_Z$. Then $q_{X \mid Z} \in \mathcal M_X$ and $q_{Z \mid X} \in \mathcal M_Z$ if and only if $q_{XZ} \in \mathcal H_{XZ}$.
	\begin{proof}
		See~\cite{arnold_conditionally_2001}, Theorem 3.

	\end{proof}
\end{theorem}
The leftward implication $\impliedby$ of this theorem is a trivial consequence of Equations~\ref{eq:likelihood-density} and~\ref{eq:posterior-density}, but the rightward implication $\implies$ is rather profound. In this direction we essentially suppose that $q_{XZ}$ has a likelihood and a posterior with forms $q_{X \mid Z}(x \mid z) \propto e^{\V s_X(x) \cdot \eprms_{X \mid Z}(z)}$ and $q_{Z \mid X}(z \mid x) \propto e^{\V s_Z(z) \cdot \eprms_{Z \mid X}(x)}$ for arbitrary functions $\eprms_{X \mid Z} \colon \mathcal Z \to \Theta_X$ and $\eprms_{Z \mid X} \colon \mathcal X \to \Theta_Z$, respectively. According to Theorem~\ref{thm:harmonium-specification}, the mere assumption that $q_{XZ}$ exists is enough to ensure that $\eprms_{X \mid Z}$ and $\eprms_{Z \mid X}$ must have the linear expressions $\eprms_{X \mid Z}(z) = \eprms_X + \iprms_{XZ} \cdot \V s_Z(z)$ and $\eprms_{Z \mid X}(x) = \eprms_Z + \iprms_{XZ} \cdot \V s_X(x)$, respectively.

When developing an LVM for a particular statistical problem, there is often an intuitive choice of exponential family structure for the posterior and likelihood. For example, if we wish to cluster count data, it is natural to model the likelihood as Poisson distributed, and the posterior as categorically distributed. Based on Theorem~\ref{thm:harmonium-specification}, this simple set of assumptions is sufficient to ensure that the proposed LVM is an exponential family harmonium.

More recent work has shown how to generalize the intrinsic characterization of harmoniums to Markov Random Fields (MRFs), which model multiple random variables with dependencies represented by a graph~\cite{yang_graphical_2015,tansey_vector-space_2015}. Although a full treatment of this extended theory is beyond the scope of this paper, the theory we develop may nevertheless be applied to such MRFs by converting the given MRF into a harmonium. In particular, if we model the random variables $X_1, \dots X_n$ as an MRF where each $X_i$ is in some chosen exponential family $\mathcal M_{X_i}$, the resulting model will be a form of higher-order harmonium, parameterized by a collection of tensors that describe the interactions between every clique of the MRF\@. This collection of tensors may be flattened into a pair of biases and an interaction matrix $\iprms_{XZ}$ between an arbitrary partition of the nodes $X = (X_{i_1}, \dots, x_{i_m})$ and $Z = (X_{j_1}, \dots, x_{j_l})$, such that the interaction $\theta_{{X_i}{X_j}}$ is zero if $X_i$ and $X_j$ are not adjacent.

\section{Conjugation for Latent Variable Models}

The likelihood $q_{X \mid Z}$ and posterior $q_{Z \mid X}$ of a harmonium density $q_{XZ}$ have simple linear structures, and are always in the exponential families $\mathcal M_X$ and $\mathcal M_Z$, respectively. In general, however, the observable density $q_X$ and prior $q_Z$ are not members of $\mathcal M_X$ and $\mathcal M_Z$, respectively. This is both a blessing and a curse, as on one hand, this allows the set of all observable densities to represent more complex distributions than the simpler set $\mathcal M_X$. On the other hand, because the prior may not be computationally tractable, various computations with harmoniums, such as sampling, learning, and inference, may also prove intractable.

\subsection{Conjugated Harmoniums}

Ideally, the set of observable densities would be more complex than $\mathcal M_X$ to ensure maximum representational power, while the priors would remain in $\mathcal M_Z$ to facilitate tractability. Perhaps surprisingly, some classes of harmoniums do indeed have this structure. In general, a prior and posterior are said to be conjugate if they have the same form. In the context of a harmonium density $q_{XZ}$, since the harmonium posterior $q_{Z \mid X} \in \mathcal M_Z$ by construction, the harmonium prior and posterior are conjugate if $q_Z \in \mathcal M_Z$.

\begin{definition}[Conjugated Harmonium]
	Where $\mathcal H_{XZ}$ is a harmonium defined by $\mathcal M_X$ and $\mathcal M_Z$, a harmonium density $q_{XZ} \in \mathcal H_{XZ}$ is conjugated if $q_Z \in \mathcal M_Z$, and the harmonium $\mathcal H_{XZ}$ is conjugated if every $q_{XZ} \in \mathcal H_{XZ}$ is conjugated.
\end{definition}

Trivially, any $q_{XZ} \in \mathcal H_{XZ}$ with parameters $\eprms_X$, $\eprms_Z$, and $\iprms_{XZ}$ is conjugated if the interactions $\iprms_{XZ} = \V 0$. This follows directly from Equation~\ref{eq:harmonium-prior}, and corresponds to the case where $q_X$ and $q_Z$ are independent --- unfortunately, by the same logic, $q_X \in \mathcal M_X$ and so such a model offers no additional representational power. With the following lemma --- the Conjugation Lemma --- we present a necessary and sufficient constraint on the parameters of $q_{XZ}$ that ensures conjugation, yet which allows for $q_X$ not to be in $\mathcal M_X$.

\begin{lemmaE}[The Conjugation Lemma]\label{thm:conjugation-lemma}
	Suppose that $\mathcal H_{XZ}$ is a harmonium defined by the exponential families $\mathcal M_X$ and $\mathcal M_Z$, and that $q_{XZ} \in \mathcal H_{XZ}$ has parameters $(\eprms_X, \eprms_Z, \iprms_{XZ})$. Then $q_{XZ} \in \mathcal H_{XZ}$ is conjugated if and only if there exists a vector $\rprms$ and a scalar $\chi$ such that
	\begin{equation}\label{eq:conjugation-equation}
		\psi_X(\eprms_X + \iprms_{XZ} \cdot \V s_Z(z)) = \V s_Z(z) \cdot \rprms + \chi,
	\end{equation}
	for any $z \in \mathcal Z$.
\end{lemmaE}
\begin{proofE}
	On one hand, if we assume that $q_Z \in \mathcal M_Z$ with parameters $\eprms^*_Z$, then
	\begin{align*}
		q_Z(z) \propto\  & e^{\eprms_Z \cdot \V s_Z(x) + \psi_X(\eprms_X + \iprms_{XZ} \cdot \V s_Z(z))} \propto e^{\eprms^*_Z \cdot \V s_Z(z)} &  & \\
		\implies         & \eprms_Z \cdot \V s_Z(z) + \psi_X(\eprms_X + \iprms_{XZ} \cdot \V s_Z(z))                                                 \\
		                 & = \eprms^*_Z \cdot \V s_Z(z) + \chi                                                                                       \\
		\implies         & \psi_X(\eprms_X + \iprms_{XZ} \cdot \V s_Z(z)) = \V s_Z(z) \cdot \rprms + \chi.
	\end{align*}
	for some $\chi$, and $\rprms = \eprms^*_Z - \eprms_Z$.

	On the other hand, if we first assume that Eq.~\ref{eq:conjugation-equation} holds, then $q_Z$ is given by
	\begin{align*}
		q_Z(z) & \propto e^{\eprms_Z \cdot \V s_Z(z) + \psi_X(\eprms_X + \iprms_{XZ} \cdot \V s_Z(z))} \\
		       & \propto e^{(\eprms_Z + \rprms) \cdot \V s_Z(z)},
	\end{align*}
	which implies that $q_Z \in \mathcal M_Z$ with parameters $\eprms_Z + \rprms$.
\end{proofE}

Most of the computations in our theory involve operations on the parameters $\rprms$ and $\chi$, and we refer to them as the conjugation parameters. We introduce two of these simpler computations with the following corollaries. The first states that marginalizing the observable variables out of a conjugated harmonium density reduces to vector addition.

\begin{corollary}\label{thm:conjugated-prior}
	Suppose that $\mathcal H_{XZ}$ is a harmonium defined by the exponential families $\mathcal M_X$ and $\mathcal M_Z$, and that $q_{XZ} \in \mathcal H_{XZ}$ with parameters $(\eprms_X, \eprms_Z, \iprms_{XZ})$ is conjugated with conjugation parameters $\rprms$ and $\chi$. Then the parameters $\eprms^*_Z$ of $q_Z \in \mathcal M_Z$ are given by
	\begin{equation}\label{eq:latent-parameters}
		\eprms^*_Z = \eprms_Z + \rprms.
	\end{equation}
\end{corollary}
\begin{proof}
	This follows from the second part of the proof of Lemma~\ref{thm:conjugation-lemma}.
\end{proof}

Assuming we can sample from densities in $\mathcal M_X$ and $\mathcal M_Z$, Corollary~\ref{thm:conjugated-prior} also shows that we may sample any conjugated $q_{XZ} \in \mathcal H_{XZ}$ by first sampling from $q_Z \in \mathcal M_Z$, and then $q_{X \mid Z} \in \mathcal M_X$. Our next corollary states that the log-partition of a conjugated harmonium density may also be tractably evaluated.

\begin{corollaryE}[]\label{thm:conjugated-log-partition}
	Suppose that $\mathcal H_{XZ}$ is a harmonium defined by the exponential families $\mathcal M_X$ and $\mathcal M_Z$, that $q_{XZ} \in \mathcal H_{XZ}$ with parameters $(\eprms_X, \eprms_Z, \iprms_{XZ})$ is conjugated with conjugation parameters $\rprms$ and $\chi$, and that $q_Z \in \mathcal M_Z$ has parameters $\eprms^*_Z = \eprms_Z + \rprms$. Then the log-partition function $\psi_{XZ}$ is given by
	\begin{equation}\label{eq:conjugated-log-partition}
		\psi_{XZ}(\eprms_X, \eprms_Z, \iprms_{XZ}) = \psi_Z(\eprms^*_Z) + \chi,
	\end{equation}
	where $\psi_Z$ is the log-partition function of $\mathcal M_Z$.
\end{corollaryE}
\begin{proofE}
	Suppose $q_{XZ} \in \mathcal H_{XZ}$ is conjugated with conjugation parameters $\rprms$ and $\chi$. Then by substituting Equation~\ref{eq:conjugation-equation} into Equation~\ref{eq:harmonium-prior}
	\begin{equation*}
		q_Z(z) = e^{\eprms_Z \cdot \V s_Z(z) + \rprms \cdot \V s_Z(z) + \chi - \psi_{XZ}(\eprms_X, \eprms_Z, \iprms_{XZ})},
	\end{equation*}
	and by Corollary~\ref{thm:conjugated-prior} we know that $q_Z(z) = e^{\eprms^*_Z \cdot \V s_Z(z) - \psi_Z(\eprms^*_Z)}$, where $\eprms^*_Z = \eprms_Z + \rprms$. Therefore
	\begin{align*}
		q_Z(z)             & = e^{\eprms^*_Z \cdot \V s_Z(z) - \psi_Z(\eprms^*_Z)}                                                       \\
		                   & = e^{\eprms_Z \cdot \V s_Z(z) + \rprms \cdot \V s_Z(z) + \chi - \psi_{XZ}(\eprms_X, \eprms_Z, \iprms_{XZ})} \\
		\iff \hspace{12pt} & \psi_{XZ}(\eprms_X, \eprms_Z, \iprms_{XZ})                                                                  \\
		                   & = (\eprms_Z + \rprms - \eprms^*_Z) \cdot \V s_Z(z) + \psi_Z(\eprms^*_Z) + \chi                              \\
		                   & =\psi_Z(\eprms^*_Z) + \chi.
	\end{align*}
\end{proofE}

The log-partition function $\psi_Z$ of $\mathcal M_Z$ will often be tractable, which means that various computations that rely on $\psi_{XZ}$, such as evaluating the observable density $q_X$ (Eq.~\ref{eq:harmonium-data-distribution}), are also tractable. In particular, given the conjugated harmonium density $q_{XZ} \in \mathcal H_{XZ}$ with natural parameters $\eprms_X$, $\eprms_Z$, and $\iprms_{XZ}$, and conjugation parameters $\rprms$ and $\chi$, we may combine Equations~\ref{eq:harmonium-data-distribution} and~\ref{eq:conjugated-log-partition} to conclude that
\begin{multline}\label{eq:conjugated-data-distribution}
	\log q_X(x) = \V s_X(x) \cdot \eprms_X
	\\ + \psi_Z(\eprms_Z + \V s_X(x) \cdot \iprms_{XZ}) - \psi_{Z}(\eprms_Z + \rprms) - \chi.
\end{multline}
We may thus use Equation~\ref{eq:conjugated-log-partition} to finesse the problem we laid out at the beginning of the section: on one hand, the observable density $q_X$ of a conjugated harmonium density (Eq.~\ref{eq:conjugated-data-distribution}) is not generally in $\mathcal M_X$, and yet on the other hand it is computable up to the computability of $\psi_Z$.

\subsection{Conjugate Prior Families}

The left hand side of Equation~\ref{eq:conjugation-equation} is simply the log-partition function of the likelihood $q_{X \mid Z}$ (Eq.~\ref{eq:likelihood-density}) of a harmonium density $q_{XZ}$. Indeed, we may generalize the Conjugation Lemma by first considering an arbitrary exponential family likelihood function $f_{X \mid Z} \colon \mathcal Z \to \mathcal M_X$ that does not presuppose a probabilistic structure over $\mathcal Z$. We then choose a family of priors $\mathcal M_Z$, so that given a particular prior $q_Z \in \mathcal M_Z$, we define the posterior $q_{Z \mid X}$ using Bayes' rule as
\begin{equation}\label{eq:bayes-rule}
	q_{Z \mid X}(z \mid x) \propto f_{X \mid Z}(x \mid z)q_Z(z).
\end{equation}
From this perspective we may consider families of conjugate priors for a fixed likelihood.

\begin{definition}[Conjugate Prior Family]
	The exponential family $\mathcal M_Z$ is a conjugate prior family for the likelihood $f_{X \mid Z} \colon \mathcal Z \to \mathcal M_X$ if the posterior $q_{Z \mid X} \in \mathcal M_Z$ for any prior $q_Z \in \mathcal M_Z$, where the posterior is defined by Bayes' rule (Eq.~\ref{eq:bayes-rule}).
\end{definition}

We may then generalize the conditions of the Conjugation Lemma to arrive at what we refer to as the Conjugation Theorem.

\begin{theoremE}[The Conjugation Theorem]\label{thm:conjugation-theorem}

	Suppose that $\mathcal M_X$ and $\mathcal M_Z$ are exponential families with minimal sufficient statistics $\V s_X$ and $\V s_Z$, respectively, and that $f_{X \mid Z} \colon \mathcal Z \to \mathcal M_X$, where $\mathcal Z$ is the sample space of $\mathcal M_Z$. Then $\mathcal M_Z$ is a conjugate prior family for the likelihood $f_{X \mid Z}$ if and only if $f_{X \mid Z}$ has the form

	\begin{equation}\label{eq:conjugating-likelihood}
		f_{X \mid Z}(x \mid z) = e^{\V s_X(x) \cdot (\eprms_X + \iprms_{XZ} \cdot \V s_Z(z)) - \V s_Z(z) \cdot \rprms - \chi},
	\end{equation}
	for some parameters $\eprms_X$, $\iprms_{XZ}$, $\chi$, and $\rprms$.

\end{theoremE}
\begin{proofE}
	For $\implies$, suppose $\mathcal M_Z$ is a conjugate prior family for $f_{X \mid Z}$. Where $q_{XZ} = f_{X \mid Z} \cdot q_Z$ for some $q_Z \in \mathcal M_Z$, the likelihood of $q_{XZ}$ is $f_{X \mid Z}$, and its posterior $q_{Z \mid X}$ is given by Bayes' rule (Eq.~\ref{eq:bayes-rule}). Since $f_{X \mid Z} \in \mathcal M_X$ and $q_{Z \mid X} \in \mathcal M_Z$ by assumption, Theorem~\ref{thm:harmonium-specification} implies that $q_{XZ}$ is in the harmonium $\mathcal H_{XZ}$ defined by $\mathcal M_X$ and $\mathcal M_Z$, which implies that $f_{X \mid Z}$ has the form of Equation~\ref{eq:likelihood-density}. Finally, since $q_Z \in \mathcal M_Z$, Equation~\ref{eq:conjugation-equation} holds according to the Conjugation Lemma~(Lm.~\ref{thm:conjugation-lemma}), and Equation~\ref{eq:conjugating-likelihood} follows directly by combining Equations~\ref{eq:likelihood-density} and~\ref{eq:conjugation-equation}.

	For $\impliedby$, if Equation~\ref{eq:conjugating-likelihood} holds, then for any $q_Z \in \mathcal M_Z$ with parameters $\eprms_Z^*$,
	\begin{align*}
		q_{Z \mid X} & (z \mid x)                                                                                                                                 \\ & \propto f_{X \mid Z}(x \mid z)q_Z(z)                                                                                                       \\
		             & = e^{\V s_X(x) \cdot \eprms_X + \V s_X(x) \cdot \iprms_{XZ} \cdot \V s_Z(z) - \V s_Z(z) \cdot \rprms - \chi}e^{\V s_Z(z) \cdot \eprms^*_Z} \\
		             & \propto e^{\V s_Z(z) \cdot (\eprms^*_Z - \rprms + \V s_X(x) \cdot \iprms_{XZ})},
	\end{align*}
	which implies $q_{Z \mid X} \in \mathcal M_Z$.
\end{proofE}

For the Conjugation Lemma we assumed that the joint density $q_{XZ}$ in question had the exponential family structure of a harmonium density. For the Conjugation Theorem, on the other hand, we merely assume that the likelihood in question $f_{X \mid Z}$ is exponential family distributed. From this it follows that $f_{X \mid Z}$ must have the form of Equation~\ref{eq:conjugating-likelihood} to support exponential family conjugate priors. This generalizes the classical results for the Bayesian estimation of (latent) exponential family parameters~\cite{diaconis_conjugate_1979,arnold_conjugate_1993} to the case where learnable parameters $\eprms_X$ and $\iprms_{XZ}$ and conjugation parameters $\V \rho$ shape how information in the observations is integrated into the posterior.

We may intuitively interpret Bayes' rule as a function defined by the likelihood, where the prior and observation are the input, and the posterior is the output. When the given likelihood satisfies Equation~\ref{eq:conjugating-likelihood}, then the posterior is simply a linear function of the natural and conjugation parameters of the likelihood, and the sufficient statistic of the observation.

\begin{corollary}\label{thm:conjugating-posterior}
	Suppose that $f_{X \mid Z}$ is defined by Equation~\ref{eq:conjugating-likelihood} with natural parameters $\eprms_X$ and $\iprms_{XZ}$ and conjugation parameters $\rprms$ and $\chi$, that $\mathcal M_Z$ is a conjugate prior family for $f_{X \mid Z}$, and that the prior $q_Z \in \mathcal M_Z$ has parameters $\eprms^*_Z$. Then the posterior $q_{Z \mid X = x} \in \mathcal M_Z$ has parameters
	\begin{equation}\label{eq:conjugated-posterior-parameters}
		\eprms_{Z \mid X}(x) = \eprms^*_Z + \V s_X(x) \cdot \iprms_{XZ} - \rprms.
	\end{equation}
	\begin{proof}
		See the proof of Theorem~\ref{thm:conjugation-theorem}.
	\end{proof}
\end{corollary}

If the prior is conjugate to the posterior, then Bayes' rule can be reapplied using the posterior as a new input, and this pattern will allow us to develop a simple recursive algorithm for Bayesian inference (Sec.~\ref{sec:conjugation-independent-observations}). Our last corollary for this section details exactly how the parameters of a conjugated harmonium relate to those of a likelihood that supports a conjugate prior family.

\begin{corollaryE}
	Let $\mathcal H_{XZ}$ be the harmonium defined by $\mathcal M_X$ and $\mathcal M_Z$. Then $q_{XZ} \in \mathcal H_{XZ}$ is a conjugated harmonium density with natural parameters $\eprms_X$, $\eprms_Z$, and $\iprms_{XZ}$ and conjugation parameters $\rprms$ and $\chi$, if and only if $q_{X \mid Z}$ satisfies Equation~\ref{eq:conjugating-likelihood} with natural parameters $\eprms_X$ and $\iprms_{XZ}$ and conjugation parameters $\rprms$ and $\chi$, and the prior $q_Z \in \mathcal M_Z$ has parameters $\eprms_Z^* = \eprms_Z + \rprms$.
\end{corollaryE}
\begin{proofE}
	For $\implies$ we apply Equation~\ref{eq:conjugation-equation} to the harmonium likelihood (Eq.~\ref{eq:likelihood-density}).

	For $\impliedby$ we multiply the definitions of the likelihood and prior and see that
	\begin{align*}
		 & q_{X \mid Z}(x \mid z) \cdot p(z) \\ & \propto e^{\V s_X(x) \cdot \eprms_X + \V s_X(x) \cdot \iprms_{XZ} \cdot \V s_Z(z) - \V s_Z(z) \cdot \rprms + \V s_Z(z) \cdot (\eprms_Z + \rprms)} \\ & \propto q_{XZ}(x,z).
	\end{align*}
\end{proofE}

\section{Example Conjugated Harmoniums}

In this section we review a variety of well-known, tractable LVMs and derive the conditions under which they are conjugated. We rely on the intrinsic characterization of harmoniums provided by Theorem~\ref{thm:harmonium-specification} to derive the exponential family form of the models we consider. In Appendix~\ref{app:parameter-transformations} we demonstrate how to reparameterize some of the LVMs we consider in terms of their more conventional parameterizations.

\subsection{Mixture Models}

A mixture model describes the statistics of observations with a weighted sum of component models. We can interpret this sum as the observable density $q_X$ of the LVM $q_{XK} = q_{X \mid K} \cdot q_X$, by equating the likelihood $q_{X \mid K=k}$ at each $k$ with one of the mixture components, and the prior $q_K$ with the weights. The posterior $q_{K \mid X = x}$ of $q_{XK}$ is then a set of weights that provides a soft classification of the observation $x$.

The $d_K$-dimensional categorical exponential family $\mathcal M_K$ contains all densities over indices from $0$ to $d_K$. It is defined by the base measure $\mu_K(k) = 1$, and the sufficient statistic given by a ``one-hot'' vector, where $\V s_K(0) = \V 0$, and for $k>0$, $s_{K,i}(k) = 1$ if $i = k$, and 0 otherwise. Because the posterior of a mixture model $q_{K \mid X}$ is a density over indices, it must be in $\mathcal M_K$. If then we assume that the likelihood $q_{X \mid K} \in \mathcal M_X$ for a chosen exponential family $\mathcal M_X$ over $X$, then $q_{XK}$ is an element of the harmonium $\mathcal H_{XK}$ defined by $\mathcal M_X$ and the categorical family $\mathcal M_K$ (Thm.~\ref{thm:harmonium-specification})\footnote{See~\ref{app:mixture-model-to-harmonium} for how to express a mixture model as a harmonium.}. Moreover, the harmonium $\mathcal H_{XK}$ is conjugated.

\begin{theoremE}\label{thm:categorical-conjugation}

	Let $\mathcal H_{XK}$ be the harmonium defined by $\mathcal M_X$ and $\mathcal M_K$, where $\mathcal M_K$ is the categorical family of dimension $d_K$. Then $\mathcal H_{XK}$ is conjugated, and the conjugation parameters of any $q_{XK} \in \mathcal H_{XK}$ are given by
	\begin{equation}
		\begin{aligned}[t]\label{eq:categorical-parameters}
			\rho_i & = \psi_X(\eprms_X + \eprms_{XK,i}) - \psi_X(\eprms_X), \\
			\chi   & = \psi_X(\eprms_X),
		\end{aligned}
	\end{equation}
	where $(\eprms_X, \eprms_K, \iprms_{XK})$ are the parameters of $q_{XK}$, and $\eprms_{XK,i}$ is the $i$th column of $\iprms_{XK}$.
\end{theoremE}
\begin{proofE}
	We must show that for any $k \in \mathcal K$, Equation~\ref{eq:conjugation-equation} is satisfied for some $\chi$ and $\rprms_K$. For $k = 0$,
	\begin{align*}
		     & \psi_X(\eprms_X + \iprms_{XK} \cdot \V s_K(k)) = \V s_K(k) \cdot \rprms_K + \chi \\
		\iff & \chi = \psi_X(\eprms_X).
	\end{align*}
	For $k > 0$,
	\begin{align*}
		     & \psi_X(\eprms_X + \iprms_{XK} \cdot \V s_K(k)) = \V s_K(k) \cdot \rprms_K + \chi \\
		\iff & \psi_X(\eprms_X + \eprms_{XK,i}) = \rho_{K,i} + \chi                             \\
		\iff & \rho_{K,i} = \psi_X(\eprms_X + \eprms_{XK,i}) - \chi.
	\end{align*}
	The conjugation parameters defined by Equations~\ref{eq:categorical-parameters} thus satisfy the Conjugation Equation (Eq.~\ref{eq:conjugation-equation}) for any $k \in \mathcal K$.
\end{proofE}

\begin{figure}[t]
	\includegraphics{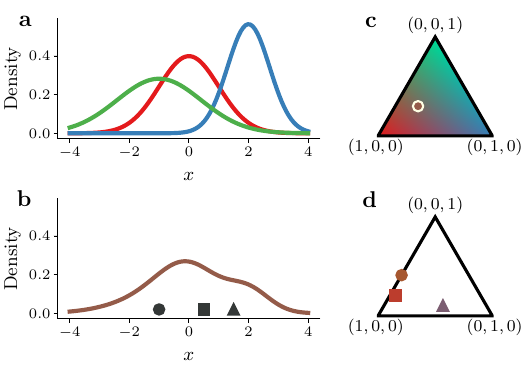}
	\caption{\emph{A mixture of normal densities.} \textbf{a:} The three component densities $q_{X \mid K=0}$ (red line), $q_{X \mid K=1}$ (blue line), and $q_{X \mid K=2}$ (green line) of a mixture of normal densities $q_{XK}$.\ \textbf{b:} The mixture density $q_X$ of $q_{XK}$ (brown line), and three observations $x^{(1)}$, (circle) $x^{(2)}$ (square), and $x^{(3)}$ (diamond).\ \textbf{c:} The exponential family $\mathcal M_K$ (coloured simplex) of categorical densities over 3 states, and the weights $(0.5,0.2,0.3)$ of $q_K$ (brown dot).\ \textbf{d:} The weights (dot colour) of the posterior densities $q_{K \mid X={x^{(1)}}}$ (circle), $q_{K \mid X={x^{(2)}}}$ (square), and $q_{K \mid X={x^{(3)}}}$ (diamond).}\label{fig:mixture-normal}
\end{figure}

In Figure~\ref{fig:mixture-normal} we demonstrate a simple mixture of normal densities $q_{XK}$. Although such a mixture model is far from novel, it exemplifies key features of conjugated harmoniums that we will generalize to more complex models. On one hand, the likelihood $q_{X \mid K} \in \mathcal M_X$ for each $k \in \{0, 1, 2\}$ (Fig.~\ref{fig:mixture-normal}\textbf{a}), yet the observable density $q_X$ is not in $\mathcal M_X$ (Fig.~\ref{fig:mixture-normal}\textbf{b}). On the other hand, both the prior (Fig.~\ref{fig:mixture-normal}\textbf{c}) and the posterior (Fig.~\ref{fig:mixture-normal}\textbf{d}) are in $\mathcal M_K$.

\subsection{Linear (Gaussian) Models}

A linear Gaussian model (LGM) $\mathcal G_{XZ}$ is a set of $(n+m)$ dimensional multivariate normal densities over $n$- and $m$-dimensional observable and latent variables $X$ and $Z$, respectively. Where $\mathcal M_X$ and $\mathcal M_Z$ are the $n$ and $m$ dimensional multivariate normal families, respectively, the analytic properties of normal densities ensures that for any $q_{XZ} \in \mathcal G_{XZ}$, both the likelihood $q_{X \mid Z}$ and observable density $q_X$ are in $\mathcal M_X$, and the posterior $q_{Z \mid X}$ and prior $q_Z$ are in $\mathcal M_Z$~\cite{bishop_pattern_2006}. Because $q_X \in \mathcal M_X$, the LGM $\mathcal G_{XZ}$ has no additional representational power over $\mathcal M_X$, and so typically we assume additional constraints on the LGM\@. For example, Factor analysis (FA) and principal component analysis (PCA) are forms of LGM that assume that $\mathcal M_X$ is the family of multivariate normals with diagonal or isotropic covariance matrices, respectively --- consequently, the observable densities for FA and PCA may have non-diagonal covariance matrices and lay outside of $\mathcal M_X$. These assumptions allow FA and PCA to effectively model high dimensional observations $X$, where an unconstrained multivariate normal model would suffer from the curse of dimensionality.

The exponential family of $n$-dimensional multivariate normal densities $\mathcal M_X$ is defined by the base measure $\mu_X(\V x) = {(2\pi)}^{-\frac{n}{2}}$ and sufficient statistic $\V s_X(\V x) = (\V x, \tril(\V x \otimes \V x))$, where $\tril(\V A)$ is the lower triangular part of the given matrix $\V A$ (this construction ensures the minimality of $\V s_X$). Similarly, we may partition the natural parameters $\eprms_X$ of any $q_X \in \mathcal M_X$ into $\eprms_X = (\eprms^m_X, \eprms^\sigma_X)$, such that $\eprms^m_X$ and $\eprms^\sigma_X$ weigh the statistics $\V x$ and $\tril(\V x \otimes \V x)$, respectively, in the exponential family form of $q_X$ (Eq.~\ref{eq:exponential-family-density}). We may also define the so-called precision matrix $\iprms_X^{\sigma}$ that satisfies $\eprms_X \cdot \V s_X(\V x) = \V x \cdot \eprms^m_X + \V x \cdot \iprms^\sigma_X \cdot \V x$ by converting $\eprms^\sigma_X$ into a lower triangular matrix $\iprms^L_X$, and letting $\iprms^\sigma_X = \frac{1}{2}(\iprms^L_X  +  \iprms^U_X)$, where $\iprms^U_X$ is the transpose of $\iprms^L_X$.

For any LGM density $q_{XZ} \in \mathcal G_{XZ}$, both $q_{X \mid Z} \in \mathcal M_X$ and $q_{Z \mid X} \in \mathcal M_Z$, and LGMs are therefore subsets of the harmonium $\mathcal H_{XZ}$ defined by $\mathcal M_X$ and $\mathcal M_Z$ (Thm.~\ref{thm:harmonium-specification})\footnote{See~\ref{app:linear-gaussian-to-harmonium} for how to reparameterize an LGM as a harmonium.}. However, this larger harmonium $\mathcal H_{XZ}$ also contains densities with interactions between the second-order statistics of $X$ and $Z$, whereas the likelihoods and posteriors of a linear Gaussian model $\mathcal G_{XZ}$ model only homoscedastic (first-order) interactions. $\mathcal G_{XZ}$ is thus the subset of $\mathcal H_{XZ}$ for which all second-order interactions $\theta_{XZ,ij}$ are 0. Moreover, the densities in this subset are conjugated.

\begin{theoremE}\label{thm:linear-gaussian-conjugation}
	Let $\mathcal H_{XZ}$ be a harmonium defined by the exponential families $\mathcal M_X$ and $\mathcal M_Z$, where $\mathcal M_X$ is the multivariate normal family, and the sufficient statistic of $\mathcal M_Z$ is given by $\V s_Z(\V z) = (\V z, \tril(\V z \otimes \V z))$. Then $q_{XZ} \in \mathcal H_{XZ}$ is conjugated if
	\begin{equation}\label{eq:linear-gaussian-constraint}
		\iprms_{XZ} = \begin{pmatrix} \iprms^{m}_{XZ} & \V 0 \\ \V 0 & \V 0 \end{pmatrix},
	\end{equation}
	with conjugation parameters $\chi$, $, \rprms^{m}$, and $\V P^{\sigma}$ given by
	\begin{equation}
		\begin{aligned}[t]\label{eq:linear-gaussian-conjugation}
			\chi          & = -\frac{1}{4} \eprms^m_X \cdot {\iprms_X^{\sigma}}^{-1} \cdot \eprms^m_X - \frac{1}{2}\log |-2 \iprms^{\sigma}_X|, \\
			\rprms^m      & = -\frac{1}{2} \iprms^{m}_{ZX} \cdot {\iprms_X^{\sigma}}^{-1} \cdot \eprms^m_X,                                     \\
			\V P^{\sigma} & = -\frac{1}{4} \iprms^{m}_{ZX} \cdot {\iprms_X^{\sigma}}^{-1} \cdot \iprms^{m}_{XZ},
		\end{aligned}
	\end{equation}
	where $\iprms^{m}_{ZX}$ is the transpose of $\iprms^{m}_{XZ}$, and $\rprms^m$ and $\V P^{\sigma}$ are the conjugation parameters of the precision-weighted means $\eprms^m_X$ and precision matrix $\iprms_Z^{\sigma}$, respectively.
\end{theoremE}
\begin{proofE}
	The log-partition function of a multivariate normal family $\mathcal M_X$ is given by
	\begin{equation}\label{eq:gaussian-log-partition}
		\psi_X(\eprms_X^{m}, \iprms_X^{\sigma}) = -\frac{1}{4}\eprms^{m}_X \cdot {\iprms_X^{\sigma}}^{-1} \cdot \eprms_X^{m} - \frac{1}{2}\log |-2\iprms_X^{\sigma}|.
	\end{equation}
	Assuming Equation~\ref{eq:linear-gaussian-constraint} holds, we may express the LHS of Equation~\ref{eq:conjugation-equation} as
	\begin{align*}
		\psi_X & (\eprms_X + \iprms_{XZ} \cdot \V s_Z(\V z))
		\\& = \psi_X(\eprms^{m}_X + \iprms^{m}_{XZ} \cdot \V z, \iprms_X^{\sigma})                                                                                                              \\
		       & = -\frac{1}{4}(\eprms^{m}_X + \iprms^{m}_{XZ} \cdot \V z) \cdot {\iprms_X^{\sigma}}^{-1} \cdot (\eprms^{m}_X + \iprms^{m}_{XZ} \cdot \V z) \\ &\hspace{2em}- \frac{1}{2}\log |-2\iprms_X^{\sigma}|,
	\end{align*}
	and the RHS as
	\begin{equation*}
		\V s_Z(\V z) \cdot \rprms + \chi = \V z \cdot \rprms^{m}_Z + \V z \cdot \V P_Z^{\sigma} \cdot \V z + \chi,
	\end{equation*}
	so that
	\begin{align*}
		\V z \cdot \rprms^{m}_Z & + \V z \cdot \V P_Z^{\sigma} \cdot \V z + \chi
		=                                                                                                                                           \\& -\frac{1}{2} \V z \cdot \iprms^{m, \trans}_{XZ} \cdot {\iprms_X^{\sigma}}^{-1} \cdot \eprms^m_X                   \\
		                        & -\frac{1}{4} \V z \cdot \iprms^{m, \trans}_{XZ} \cdot {\iprms_X^{\sigma}}^{-1} \cdot \iprms^{m}_{XZ} \cdot \V z   \\
		                        & -\frac{1}{4} \eprms^m_X \cdot {\iprms_X^{\sigma}}^{-1} \cdot \eprms^m_X - \frac{1}{2}\log |-2 \iprms^{\sigma}_X|,
	\end{align*}
	which is clearly solved by Equations~\ref{eq:linear-gaussian-conjugation}.
\end{proofE}

We derived Theorem~\ref{thm:linear-gaussian-conjugation} to establish the conditions for conjugated LGMs, yet these conditions delineate a larger class of conjugated harmonium. On one hand, the observable exponential family $\mathcal M_X$ must be multivariate normal, as according to the Conjugation Lemma (Thm.~\ref{thm:conjugation-lemma}), the form of the log-partition function $\psi_X$ of $\mathcal M_X$ determines the form of the conjugation parameters in Equations~\ref{eq:linear-gaussian-conjugation}, and the log-partition function of the multivariate normal family has a particularly simple form~(Eq.~\ref{eq:gaussian-log-partition}). On the other hand, the condition on the latent exponential family $\mathcal M_Z$ is only that the sufficient statistic is given by $\V s_Z(\V z) = (\V z, \tril(\V z \otimes \V z))$, so that theoretically we may vary the base measure $\mu_Z$ and sample space $\mathcal Z$ of $\mathcal M_Z$ and continue to satisfy the conditions of the theorem. One example of such an $\mathcal M_Z$ is the Boltzmann machine~\cite{ackley_learning_1985}, which models the second-order statistics between binary random variables, typically called binary neurons. Another example is the Riemann-Theta Boltzmann machine~\cite{krefl_riemann-theta_2020}, which extends the Boltzmann machine sample space $\mathcal Z = {\{0,1\}}^m$ to the space of countable vectors $\mathcal Z = \mathbb N^m$.

The conjugated harmonium defined by the family of multivariate normals $\mathcal M_X$ and a Boltzmann machine $\mathcal M_Z$ is known as a Gaussian-Boltzmann machine, and has been previously explored for modelling the joint density of continuous stimuli and neural recordings~\cite{gerwinn_joint_2009} --- here we demonstrate how they can serve as tractable latent variable models~(Fig.~\ref{fig:gaussian-boltzmann}). We consider a Gaussian-Boltzmann harmonium $\mathcal G_{XZ}$ defined by the bivariate normal family $\mathcal M_X$ and Boltzmann machine $\mathcal M_Z$ with 6 neurons, restricted to the densities that satisfy Equation~\ref{eq:linear-gaussian-constraint}. We learn $q_{XZ} \in \mathcal G_{XZ}$ by fitting the model to a synthetic dataset generated from two, noisy concentric circles (Fig.~\ref{fig:gaussian-boltzmann}\textbf{a}). Analogous to a mixture model, we consider a set of ``component'' densities by evaluating the likelihood $q_{X \mid Z=\V z^{(i)}}$ at the one-hot vectors $\V z^{(i)}$ for each neuron $i$, such that $z^{(i)}_j = 1$ if $i = j$, and 0 otherwise (Fig.~\ref{fig:gaussian-boltzmann}\textbf{a}). Although each $q_{X \mid Z=\V z^{(i)}}$ has the same bivariate normal shape, the observable density $q_X$ is not in the family of bivariate normals, and successfully captures the concentric circle structure (Fig.~\ref{fig:gaussian-boltzmann}\textbf{b}). The correlation matrix of the prior $q_Z \in \mathcal M_Z$ reveals how the population of neurons encodes the concentric circles (Fig.~\ref{fig:gaussian-boltzmann}\textbf{c}). Similarly, the moment matrices $\E_Q[Z \otimes Z \mid X = \V x^{(i)}]$ of the posterior $q_{Z \mid X = \V x^{(i)}} \in \mathcal M_Z$ at 4 example observations $\V x^{(i)}$ demonstrates how $q_{XZ}$ encodes each point within the concentric circles by activating and correlating certain subsets of neurons, which mixes and distorts the components $q_{X \mid Z=\V z^{(i)}}$ (Fig.~\ref{fig:gaussian-boltzmann}\textbf{d}). We cover how we trained this model in Section~\ref{sec:training-harmoniums}.

\begin{figure}[t]
	\begin{center}
		\includegraphics{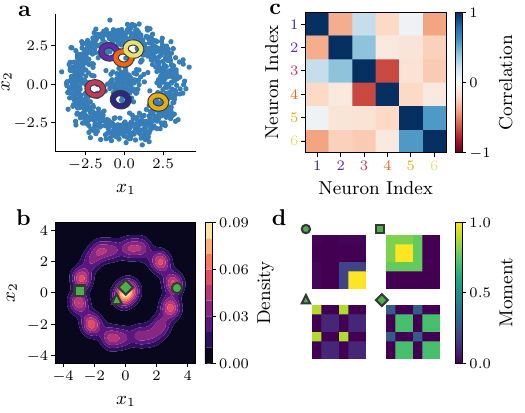}
	\end{center}
	\caption{\emph{A Gaussian-Boltzmann model.}\ \textbf{a:} Training data (blue dots) used to learn a Gaussian-Boltzmann density $q_{XZ}$, and confidence ellipses (coloured circles) of the likelihoods $q_{X \mid Z=\V z^{(i)}}$ given one-hot vectors $\V z^{(i)}$ (neuron index indicated by colour).\ \textbf{b:} Observable density $q_Z$, and example observations $\V x^{(1)}, \dots, \V x^{(4)}$ (green shapes).\ \textbf{c:} Correlation matrix of the prior $q_Z$.\ \textbf{d:} Moment matrices $\E_Q[Z \otimes Z \mid X = \V x^{(i)}]$ of the posteriors for each example observation (labelled by green shape).}\label{fig:gaussian-boltzmann}
\end{figure}

We note that an exponential family $\mathcal M_Z$ with a sufficient statistic $\V s_Z$ that is a strict subset of $(\V z, \tril(\V z \otimes \V z))$ does not in general satisfy the conditions of Theorem~\ref{thm:linear-gaussian-conjugation}. Consider, for example, the family of isotropic normal distributions defined by the sufficient statistic $\V s_Z(\V x) = (\V x, \V x \odot \V x)$, where $\odot$ is the element-wise vector product, or the family of restricted Boltzmann machines, which partitions the binary neurons of a Boltzmann machine into two groups, and assumes there are no interactions between binary neurons within the same group. Because the form of the restricted sufficient statistic $\V s_Z$ determines the form of the natural parameters $\eprms_Z$ of any $q_Z \in \mathcal M_Z$, the form of $\eprms_Z$ will be a subset of the form of the unrestricted conjugation parameters $\rprms = (\rprms^m, \V P^{\sigma})$ in Equations~\ref{eq:linear-gaussian-conjugation}. As such, formulae that rely on adding them (e.g. Equations~\ref{eq:latent-parameters} or~\ref{eq:conjugated-posterior-parameters}) will be inconsistent. In essence, the manifolds of isotropic normals or restricted Boltzmann machines are not complex enough to serve as conjugate prior families for likelihoods that are multivariate normal distributed.

\subsection{Probabilistic Population Codes}

Neuroscientists often model the spiking activity of populations of neurons as random vectors of counts $N = (N_1, \dots, N_{d_N})$, where each $N_i$ is the spike count of a single neuron. These spike counts may also depend on a stimulus or environmental variable $Z$, and theories of probabilistic population coding use the framework of Bayesian inference to explain how the neural activity $N$ encodes and processes information about $Z$. In particular, researchers have found that, under certain conditions, stimulus-dependent models with Poisson-distributed spike counts can support optimal Bayesian inference~\cite{ma_bayesian_2006,beck_marginalization_2011}. As we next show, these conditions are special cases of the Conjugation Lemma (Thm.~\ref{thm:conjugation-lemma}).

The family of Poisson distributions $\mathcal P_{N_i}$ is defined by the sufficient statistic $s_{N_i}(n) = n$ and base measure $\mu_{N_i}(n) = \frac{1}{n!}$. By extension, the family of independent Poisson product distributions $\mathcal M_N$ is defined by the sufficient statistic $\V s_N(\V n) = \V n$ and base measure $\mu_N(\V n) = {\big (\prod_{i=1}^{d_N} n_i! \big )}^{-1}$. Every density $q_N \in \mathcal M_N$ can be factored into $q_N = q_{N_1} \cdots q_{N_{d_N}}$ where each $q_{N_i} \in \mathcal P_{N_i}$ is a Poisson distribution. In the context of neuroscience a probabilistic population code for a stimulus $Z$ is essentially a joint density $q_{NZ}$. We refer to the mean spike-count $\E_Q[N_i]$ of $N_i$ under $Q$ as the firing rate, and its stimulus dependent firing-rate $\E_Q[N_i \mid Z]$ as its tuning curve. Population codes are typically specified so that the likelihood $q_{N \mid Z}$ of a population code $q_{NZ}$ is in the Poisson product family $\mathcal M_N$, and the posterior $q_{Z \mid N}$ is in some chosen exponential family $\mathcal M_Z$ over the stimulus. Such a population code $q_{XZ}$ is thus an element of the harmonium $\mathcal H_{NZ}$ defined by $\mathcal M_N$ and $\mathcal M_Z$ (Thm.~\ref{thm:harmonium-specification}). Moreover, under certain conditions on the sum of tuning curves, such population codes are approximately conjugated.

\begin{theoremE}\label{thm:neural-conjugation}
	Let $\mathcal H_{NZ}$ be a harmonium defined by $\mathcal M_N$ and $\mathcal M_Z$, where $\mathcal M_N$ is the Poisson product family. Then $q_{NZ} \in \mathcal H_{NZ}$ is conjugated with conjugation parameters $\rprms$ and $\chi$ if it satisfies
	\begin{equation}\label{eq:tuning-curve-conjugation}
		\sum_{i=1}^{d_N}\E_Q[N_i \mid Z = z] = \rprms \cdot \V s_Z(z) + \chi.
	\end{equation}
\end{theoremE}
\begin{proofE}
	Firstly, note that the log-partition function of a Poisson exponential family $\mathcal M_{N_i}$ is given by $\psi_{N_i}(\theta_{N_i}) = e^{\theta_{N_i}}$, which implies $\psi_{N_i}(\theta_{N_i}) = \partial_{\theta_{N_i}}\psi_{N_i}(\theta_{N_i}) = \E_Q[N_i]$. Moreover, the independence of the $N_i$ implies that $\psi_N(\eprms_N) = \sum_{i=1}^{d_N} \psi_{N_i}(\theta_{N,i})$, and therefore that $\psi_N(\eprms_N) = \sum_{i=1}^{d_N} \E_Q[N_i]$.

	Now, suppose $q_{NZ} \in \mathcal H_{NZ}$ has natural parameters $(\eprms_N, \eprms_Z, \iprms_{NZ})$. Then the likelihood $q_{N \mid Z = z} \in \mathcal M_N$ has natural parameters $\eprms_N + \iprms_{NZ} \cdot \V s_Z(z)$, so that $\psi_N(\eprms_N + \iprms_{NZ} \cdot \V s_Z(z)) = \sum_{i=1}^{d_N}\E_Q[N_i \mid Z = z]$. Therefore, if Equation~\ref{eq:tuning-curve-conjugation} holds,
	\begin{align*}
		\rprms_N \cdot \V s_N(\V n) + \chi & = \sum_{i=1}^{d_N}\E_Q[N_i \mid Z = z] \\&= \psi_N(\eprms_N + \iprms_{NZ} \cdot \V s_Z(z)),
	\end{align*}
	and therefore by the Conjugation Lemma (Thm.~\ref{thm:conjugation-lemma}), $q_{NZ}$ is conjugated with conjugation parameters $\rprms$ and $\chi$.
\end{proofE}

The canonical probabilistic population code has Gaussian tuning curves that sum to a constant~\cite{ma_bayesian_2006}, and Equation~\ref{eq:conjugation-equation} generalizes this constraint by allowing the sum to depend on the sufficient statistics of $z$. Unlike Theorems~\ref{thm:categorical-conjugation} and~\ref{thm:linear-gaussian-conjugation}, the result of Theorem~\ref{thm:neural-conjugation} does not specify a manifold of probability densities that exactly satisfy Equation~\ref{eq:conjugation-equation}. Nevertheless, depending on the choice of $\mathcal M_Z$, tuning curves often have simple (e.g.\ von Mises or Gaussian) shapes that serve well as basis functions, so that ensuring that the sum of the tuning curves is an affine function of $\V s_Z(z)$ is easily achievable with a sufficient numbers of model neurons.

We demonstrate conjugated population codes with a model of how the spike counts $N_1, \dots, N_8$ of 8 neurons can encode the location of an oriented stimulus (Fig.~\ref{fig:population-code-von-mises}). We model the joint density of $N$ and $Z$ with the harmonium $\mathcal M_{NZ}$ defined by the Poisson product family $\mathcal M_N$ with $d_N = 8$ Poisson neurons, and the von Mises family $\mathcal M_Z$, which is defined by the sufficient statistic $\V s_Z = (\cos z, \sin z)$, and base measure $\mu_Z(z) = \frac{1}{2\pi}$. We choose a population code $q_{NZ} \in \mathcal H_{NZ}$, and demonstrate that it is approximately conjugated by fitting a function of the form $\rho_1 \cos z + \rho_2 \sin z + \chi$ to the sum of its tuning curves (Fig.~\ref{fig:population-code-von-mises}\textbf{a}). We also find that the observable density $q_N$ is not an element of $\mathcal M_N$, and rather describes a regular pattern of correlations between the neurons (Fig.~\ref{fig:population-code-von-mises}\textbf{b}). Nevertheless, the prior $q_Z$ is in $\mathcal M_Z$ (Fig.~\ref{fig:population-code-von-mises}\textbf{c}). Finally, given spike counts generated in response to a stimulus, we use the posterior $q_{Z \mid N}$ to decode an accurate von Mises density over stimulus orientation, weighed against prior beliefs about stimulus orientation (Fig.~\ref{fig:population-code-von-mises}\textbf{d}).

\begin{figure}[t]
	\includegraphics{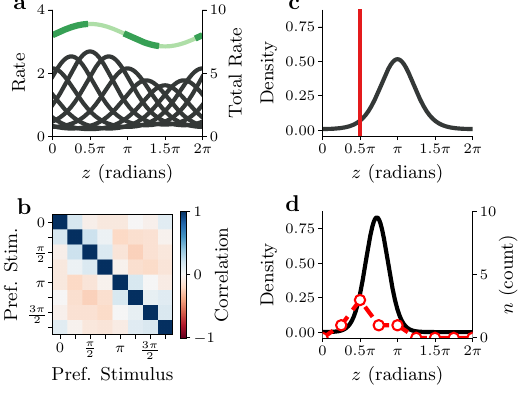}
	\caption{\emph{Conjugation with population codes.} The population code $q_{NZ}$ is a density over 8 spike counts $N_i$ and an oriented stimulus $Z$.\ \textbf{a:} Tuning curves $\E_Q[N_i \mid Z]$ (black lines), their sum (light green line), and fit of Equation~\ref{eq:tuning-curve-conjugation} (dashed green line) to the sum.\ \textbf{b:} Correlation matrix of $q_N$, with neuron identified by tuning curve peak (preferred stimulus).\ \textbf{c:} Population code prior $q_Z$ (black line), and example stimulus $z^{(1)}$ (red line).\ \textbf{d:} Example spike count vector $N^{(1)} \sim q_{N \mid Z = z^{(1)}}$ (red line-dots), and posterior $q_{Z \mid N = N^{(1)}}$ (black line).}\label{fig:population-code-von-mises}
\end{figure}

\subsection{Bayesian Parameter Estimation}

\begin{figure*}[t]
	\includegraphics{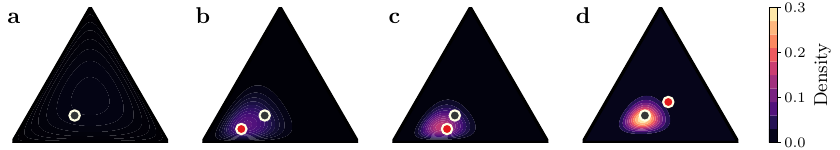}
	\caption{\emph{Inference over a categorical distribution.} \textbf{a:} The true weights of a categorical distribution (black dot) and the Dirichlet prior density (contours).\ \textbf{b-d:} The true weights of a categorical distribution (black dot) and the Dirichlet posterior density (contours) after 10 (\textbf{b}), 20 (\textbf{c}), and 30 (\textbf{d}) observations, as well as the centre of mass of each set of 10 observations (red dots).
	}\label{fig:categorical-inference}
\end{figure*}

The final conjugated harmonium we consider returns us to the conceptual origins of conjugate priors, in which the latent variables are the parameters of an exponential family model~\cite{diaconis_conjugate_1979,arnold_conjugate_1993}. A classic result in Bayesian parameter estimation is that any exponential family $\mathcal M_X$ has a conjugate prior family $\mathcal M_Z$ with sufficient statistic $s_Z(z) = (z, \psi_X(z))$, where $\psi_X$ is the log-partition function of $\mathcal M_X$.

\begin{theoremE}\label{thm:bayesian-naturals}
	Suppose $\mathcal M_X$ is a $d_X$-dimensional exponential family, and let $\mathcal M_Z$ be the $d_{X} + 1$-dimensional exponential family with sample space $\mathcal Z = \Theta_X$, and sufficient statistic given by $\V s_Z(z) = (z, \psi_X(z))$. Then $q_{XZ} \in \mathcal H_{XZ}$ is conjugated if $\eprms_X = \V 0$, and

	\begin{equation}\label{eq:pseudo-identity}
		\iprms_{XZ} =
		\begin{bmatrix}
			1      & 0      & \cdots & 0      & 0      \\
			0      & 1      & \cdots & 0      & 0      \\
			\vdots & \vdots & \ddots & \vdots & \vdots \\
			0      & 0      & \cdots & 1      & 0      \\
		\end{bmatrix}
	\end{equation}
\end{theoremE}
\begin{proofE}
	Suppose $\eprms_X = \V 0$ and $\iprms_{XZ}$ is given by Eq.~\ref{eq:pseudo-identity}. Then
	\begin{equation*}
		\psi_X(\eprms_X + \iprms_{XZ} \cdot \V s_Z(z)) = \psi_X(z) = \V s_Z(z) \cdot \rprms + \chi,
	\end{equation*}
	where $\rprms = (0,0,\dots,0,1)$ and $\chi = 0$. Therefore $q_{XZ}$ is conjugated by Lemma~\ref{thm:conjugation-lemma}.
\end{proofE}

Although this provides a general solution to Bayesian parameter estimation, the resulting conjugate prior families are often intractable, and many exponential families have alternative conjugate priors with more tractable forms. We explore one example of such a conjugate prior in in the following section.

\section{Conjugation and Recursive Inference}\label{sec:conjugation-independent-observations}

Consider a sequence of observable variables $X_1, \dots, X_n$ that are conditionally independent given the latent variable $Z$, such that the conditional density of the observable variables given the latent variable is $p_{X_1, \dots, X_n \mid Z} = \prod_{i=1}^n p_{X_i \mid Z}$. Then we may express $p_{Z \mid X_1, \dots, X_n}$ in the recursive form
\begin{multline}
	p_{Z \mid X_1, \dots, X_n}(z \mid x_1, \dots, x_n) \\ \propto p_{X_n \mid Z}(x_n \mid z) p_{Z \mid X_1, \dots, X_{n-1}}(z \mid x_1, \dots, x_{n-1}),
	\label{eq:recursive-bayesian-inference}
\end{multline}
where for the base-case $n=1$, we define $p_{Z \mid X_1, \dots, X_{n-1}}$ as the prior $p_Z$. Observe that this recursive relation is a single application of Bayes' rule (Eq.~\ref{eq:bayes-rule}) given the prior $p_{Z \mid X_1, \dots, X_{n-1}}$, so that Equation~\ref{eq:recursive-bayesian-inference} reduce sequential inference to iterative applications of Bayes' rule.

In general, however, each application of Bayes' rule may increase the complexity of the beliefs, and ultimately produce an intractable posterior. We may avoid this by assuming that there is a single exponential family $\mathcal M_Z$ that is a conjugate prior family for each likelihood $p_{X_i \mid Z}$, which ultimately ensures that the posterior over $Z$ given all the observations is also an element of $\mathcal M_Z$.

\begin{theoremE}
	Suppose that the sequence of observable variables $X_1, \dots, X_n$ are conditionally independent given the latent variable $Z$, that $\mathcal M_Z$ is a conjugate prior family for each $p_{X_i \mid Z}$, and that $p_Z \in \mathcal M_Z$ with parameters $\eprms^*_Z$. Moreover, suppose that each $p_{X_i \mid Z} \in \mathcal M_X$ is defined by natural parameters $\eprms_{X_i}$ and $\iprms_{X_i Z}$, and conjugation parameters $\rprms_i$ and $\chi_i$ according Equation~\ref{eq:conjugating-likelihood}. Then the parameters $\eprms_{Z \mid X_1, \dots, X_n}$ of the posterior $p_{Z \mid X_1, \dots, X_n} \in \mathcal M_Z$ are given by
	\begin{multline}\label{eq:recursive-bayesian-parameters}
		\eprms_{Z \mid X_1, \dots, X_n}(x_1, \dots, x_n) =\\ \eprms^*_Z + \sum_{i=1}^n \V s_{X_i}(x_i) \cdot \iprms_{X_i Z} - \rprms_i.
	\end{multline}
\end{theoremE}
\begin{proofE}
	For $n=1$, the prior is $p_Z$ and its parameters are $\eprms^*_Z$. Given $p_{X_1 \mid Z}$ with parameters $\eprms_{X_1}$ and $\iprms_{X_1 Z}$, Corollary~\ref{thm:conjugating-posterior} implies that the parameters of $p_{Z \mid X_1}(x_1)$ are $\eprms_{Z \mid X_1} = \eprms^*_Z + \V s_{X_1}(x_1) \cdot \iprms_{X_1 Z} - \rprms_1$.

	For $n$ assume that the parameters $\eprms_{Z \mid X_1, \dots, X_n}$ of $p_{Z \mid X_1, \dots, X_n}$ are given by Equation~\ref{eq:recursive-bayesian-parameters}. Where $p_{X_{n+1} \mid Z}$ has parameters $\eprms_{X_{n+1}}$ and $\iprms_{X_{n+1} Z}$, Corollary~\ref{thm:conjugating-posterior} implies that the parameters of $p_{Z \mid X_1, \dots, X_{n+1}}(x_1, \dots, x_{n+1})$ are
	\begin{align*}
		\eprms_{Z \mid X_1, \dots, X_{n+1}} & = \eprms_{Z \mid X_1, \dots, X_n}(x_1, \dots, x_n)                               \\ &\hspace{2em} + \V s_{X_{n+1}}(x_{n+1}) \cdot \iprms_{X_{n+1} Z} - \rprms_{n+1} \\
		                                    & = \eprms^*_Z + \sum_{i=1}^{n+1} \V s_{X_i}(x_i) \cdot \iprms_{X_i Z} - \rprms_i.
	\end{align*}
	Therefore, by induction, Equation~\ref{eq:recursive-bayesian-parameters} holds for any $n$.
\end{proofE}

By combining the general solution to Bayesian parameter estimation (Thm.~\ref{thm:bayesian-naturals}) with the above formula for recursive Bayesian inference (Eq.~\ref{eq:recursive-bayesian-parameters}) we may conclude that the posterior for the general solution to Bayesian parameter estimation has parameters $\eprms_{Z \mid X_1, \dots, X_n}(x_1, \dots, x_n) = \eprms_Z + \sum_{i=1}^n (\V s_{X_i}(x_i), 1)$, which indeed matches the established formula~\cite{murphy_probabilistic_2023}. As a more practical example, suppose that $\mathcal H_{XZ}$ is the harmonium defined by the categorical family $\mathcal M_X$ and the Dirichlet family $\mathcal M_Z$. It is easy to check that $q_{XZ} \in \mathcal H_{XZ}$ is conjugated if $\eprms_X = \V 0$ and
\begin{equation}
	\iprms_{XZ} =
	\begin{bmatrix}
		-1     & 1      & \cdots & 0      \\
		\vdots & \vdots & \ddots & \vdots \\
		-1     & 0      & \cdots & 1      \\
	\end{bmatrix}
\end{equation}
with conjugation parameters $\rprms = (-1,0,\dots,0)$ and $\chi = 0$. By simulating this model we find that, given an increasingly large sequence of observations from a categorical distribution, the posterior Dirichlet density concentrates around the weights of the categorical distribution~(Fig.~\ref{fig:categorical-inference}).

\section{Fitting Conjugated Harmoniums}\label{sec:training-harmoniums}

We formulate the problem of fitting a model to data within the framework of log-likelihood maximization, or equivalently, cross-entropy minimization --- in this paper we favour the terminology of cross-entropy minimization to avoid any confusion of our training objective with the likelihood $q_{X \mid Z}$ of an LVM density $q_{XZ}$. In general, where $q_X$ is a parametric density over observable variable $X$ with parameters $\eprms$, the cross-entropy objective is
\begin{equation}\label{eq:cross-entropy}
	\frac{1}{n} \sum_{i=1}^n \mathcal L(X^{(i)}, \V \theta) = -\frac{1}{n} \sum_{i=1}^{n} \log q_X(X^{(i)}),
\end{equation}
where $X^{(1)}, \dots, X^{(n)}$ is an independent sample from $P_X$, and $\mathcal L(X^{(i)}, \V \theta) = -\log q_X(X^{(i)})$ is the pointwise cross-entropy.

When $q_X$ is an element of an exponential family $\mathcal M_X$ with parameters $\eprms_X$, the pointwise cross-entropy reduces to $\mathcal L(X^{(i)}, \V \theta) = -\V s_X(X^{(i)}) \cdot \eprms_X + \psi_X(\eprms_X)$. Moreover, $\frac{1}{n}\sum_{i=1}^n \mathcal L_X$ is a convex function on the natural parameter space $\Theta_X$ of $\mathcal M_X$, and its optimum is given by $\V \tau_X^{-1}(-\frac{1}{n} \sum_{i=1}^{n} \V s_X(X^{(i)}))$, where $\V \tau_X^{-1}$ is the backward mapping of $\mathcal M_X$~\cite{wainwright_graphical_2008}.

\subsection{Expectation-Maximization}

For an exponential family LVM over observables $X$ and latent variables $Z$, the observable densities will not typically be exponential family distributed. In this case we may employ EM, which is one of the most widely applied algorithms for minimizing the cross-entropy of an LVM~\cite{dempster_maximum_1977}. For an arbitrary exponential family $\mathcal M_{XZ}$ over $X$ and $Z$, defined by the sufficient statistic $\V s_{XZ}$ and base measure $\mu_{XZ}$, the E-Step of EM is to calculate the conditional sufficient statistics
\begin{equation}\label{eq:expectation-step}
	\mprms_{XZ}^{(i)} = \E_Q[\V s_{XZ}(X, Z) \mid X = X^{(i)}],
\end{equation}
for a given $q_{XZ} \in \mathcal M_{XZ}$~\cite{wainwright_graphical_2008}. If the exponential family is a harmonium $\mathcal H_{XZ}$ defined by $\mathcal M_X$ and $\mathcal M_Z$, and $q_{XZ} \in \mathcal H_{XZ}$ with parameters $\eprms_X$, $\eprms_Z$, and $\iprms_{XZ}$, then the conditional sufficient statistics $\mprms^{(i)}_{XZ} = (\mprms^{(i)}_X, \mprms^{(i)}_Z, \V H^{(i)}_{XZ})$, are given by
\begin{equation}\label{eq:harmonium-expectation-step}
	\begin{aligned}[t]
		\mprms^{(i)}_X  & = \V s_X(X^{(i)}),                                         \\
		\mprms^{(i)}_Z  & = \V \tau_Z(\eprms_Z + \V s_X(X^{(i)}) \cdot \iprms_{XZ}), \\
		\V H^{(i)}_{XZ} & =  \mprms^{(i)}_X \otimes \mprms^{(i)}_Z,
	\end{aligned}
\end{equation}
for every $i$, where $\V \tau_Z$ is the forward mapping of $\mathcal M_Z$.

The objective of the M-Step of EM for an exponential family LVM is to minimize the cross-entropy loss (Eq.~\ref{eq:cross-entropy}) of the joint density $q_{XZ}$, where we use the conditional sufficient statistics of Equation~\ref{eq:expectation-step} to fill in the missing data. The M-Step objective function is thus
\begin{multline}
	\frac{1}{n} \sum_{i=1}^n \mathcal L(\mprms^{(i)}_{XZ},\eprms_{XZ}) =\\ -\frac{1}{n} \sum_{i=1}^{n} \mprms^{(i)}_{XZ} \cdot \eprms_{XZ} - \psi_{XZ}(\eprms_{XZ}).
\end{multline}
This is again a convex optimization problem, and its solution is given by the backward mapping, so that
\begin{equation}\label{eq:maximization-step}
	\argmin_{\eprms_{XZ}} \frac{1}{n} \sum_{i=1}^n\mathcal L(\mprms^{(i)}_{XZ},\eprms_{XZ}) = \V \tau_{XZ}^{-1}(\frac{1}{n} \sum_{i=1}^{n} \mprms_{XZ}^{(i)} ),
\end{equation}
where $\V \tau_{XZ}^{-1}$ is the backward mapping of $\mathcal M_{XZ}$ (or $\mathcal H_{XZ}$ in the harmonium case).  The EM algorithm thus minimizes the cross-entropy by iteratively updating the model parameters with Equations~\ref{eq:expectation-step} and~\ref{eq:maximization-step}. As such, we can implement exact EM for exponential family LVMs as long as we can evaluate $\V \tau_Z$ and $\V \tau^{-1}_{XZ}$.

Theorem~\ref{thm:conjugated-log-partition} provides a simpler expression for the log-partition function $\psi_{XZ}$ of a conjugated harmonium $\mathcal M_{XZ}$, and thereby a path for evaluating $\V \tau^{-1}_{XZ}$. Suppose $q_{XZ} \in \mathcal H_{XZ}$ is a conjugated harmonium density with natural parameters $\eprms_{XZ} = (\eprms_X, \eprms_Z, \iprms_{XZ})$, and conjugation parameters $\rprms$ and $\chi$, and let us consider $\rprms(\eprms_X, \iprms_{XZ})$ and $\chi(\eprms_X, \iprms_{XZ})$ as functions of the natural parameters. We may then combine the forward mapping $\V \tau_{XZ}(\eprms_{XZ}) = \partial_{\eprms_{XZ}}\psi_{XZ}(\eprms_{XZ})$ with Equation~\ref{eq:conjugated-log-partition} to express the $i$th mean parameter of $q_{XZ}$ as
\begin{multline}
	\eta_{XZ,i}  = \partial_{\theta_{XZ,i}} \chi(\rprms_X, \iprms_{XZ}) + \\ \V \tau_Z(\eprms_Z + \rprms(\eprms_X, \iprms_{XZ})) \cdot \partial_{\theta_{XZ,i}} \rprms(\eprms_X, \iprms_{XZ}),
\end{multline}
which for the latent mean parameters simplifies to
\begin{equation}
	\mprms_Z  = \V \tau_Z(\eprms_Z + \rprms(\eprms_X, \iprms_{XZ})).
\end{equation}
We may then derive $\V \tau^{-1}_{XZ}$ by inverting the forward mapping. This inversion is in fact possible for mixture models and linear Gaussian models, but certainly not in general, and so we must derive alternative approaches when the harmonium is less analytically tractable.

\subsection{Gradient Descent}

An alternative training algorithm that avoids computing $\V \tau^{-1}_{XZ}$ is cross-entropy gradient descent (CE-GD). To implement CE-GD for $q_{XZ} \in \mathcal H_{XZ}$ with natural parameters $(\eprms_X, \eprms_Z, \iprms_{XZ})$, we first compute the conditional expectations (Eq.~\ref{eq:harmonium-expectation-step}), and then the model expectations $(\mprms_X, \mprms_Z, \V H_{XZ}) = \V \tau_{XZ}(\eprms_X, \eprms_Z, \iprms_{XZ})$ using the forward mapping $\V \tau_{XZ}$. The gradients of $\mathcal L(X^{(i)}, \eprms_{XZ})$ with respect to the natural parameters of $q_{XZ}$ are then given by
\begin{equation}
	\begin{aligned}[t]\label{eq:cross-entropy-gradients}
		\partial_{\eprms_X} \mathcal L(X^{(i)},\eprms_{XZ})    & = \mprms_X - \mprms^{(i)}_X ,   \\
		\partial_{\eprms_Z} \mathcal L(X^{(i)},\eprms_{XZ})    & =  \mprms_Z - \mprms^{(i)}_Z,   \\
		\partial_{\iprms_{XZ}} \mathcal L(X^{(i)},\eprms_{XZ}) & =  \V H_{XZ} - \V H^{(i)}_{XZ}.
	\end{aligned}
\end{equation}
Using these gradients we may then iteratively update the natural parameters of $\eprms_X$, $\eprms_Z$, and $\iprms_{XZ}$, using any standard gradient pursuit algorithm such as Adam~\cite{kingma_adam_2014}, and by following the average gradient or using a batch strategy.

Both EM and CE-GD rely on Equations~\ref{eq:expectation-step}, and differ in either solving the M-Step with the backward mapping (Eq.~\ref{eq:maximization-step}) or following the cross-entropy gradients (Eq.~\ref{eq:cross-entropy-gradients}). Yet the M-Step is a convex optimization problem, and we could also approximate its solution with gradient descent. We thus define expectation-maximization gradient descent (EM-GD) as the algorithm that approximates the M-Step by pursuing the cross-entropy gradients in Equations~\ref{eq:cross-entropy-gradients} and recomputing $\mprms_{XZ}$ at every step, while holding $\mprms^{(i)}_{XZ}$ fixed. After convergence to the M-Step optimum, we recompute $\mprms^{(i)}_{XZ}$ and begin another iteration of EM-GD\@.

Both CE-GD and EM-GD can effectively train LVMs while avoiding the evaluation of the backward mapping $\V \tau^{-1}_{XZ}$. On one hand, CE-GD descends the cross-entropy objective directly, and avoids excessive computations that might arise for EM-GD while near the optimum of an M-Step. On the other hand, EM-GD can reuse the conditional expectations $\mprms^{(i)}_{XZ}$ over multiple gradient steps, which may be desirable if they are computationally expensive to evaluate. In addition, EM and EM-GD solve a sequence of convex optimization problems, rather than solving a single non-linear optimization problem like CE-GD, which can in practice improve the stability of the learning algorithm.

\subsection{Monte Carlo Estimation}

For cases where we cannot compute the conditional or model expectations of the sufficient statistics analytically, we can often estimate them through sampling and Monte Carlo methods. Markov chain algorithms like Gibbs sampling can generate approximate samples from LVMs that we can then use to estimate their necessary expectations. Gibbs sampling, however, can be prohibitively slow, and algorithms such as contrastive divergence~\cite{hinton_training_2002,welling_exponential_2005} were developed to approximate the cross-entropy gradient while avoiding full Gibbs sampling. Nevertheless, even contrastive divergence can still necessitate large numbers of sampling cycles, and its convergence properties can be difficult to characterize~\cite{bengio_justifying_2009, sutskever_convergence_2010}.

In contrast, we can directly generate exact samples from conjugated harmoniums. Consider a harmonium $\mathcal H_{XZ}$ defined by $\mathcal M_X$ and $\mathcal M_Z$, and suppose that there are efficient algorithms for sampling from the densities in $\mathcal M_X$ and $\mathcal M_Z$. If $q_{XZ} \in \mathcal H_{XZ}$ is a conjugated harmonium density, then its prior $q_Z \in \mathcal M_Z$, and we can generate an exact sample point $(X,Z) \sim q_{XZ}$ by sampling $Z \sim q_Z$ followed by $X \sim q_{X \mid Z = Z}$. We thus propose cross-entropy Monte Carlo gradient descent (CE-MCGD) and EM Monte Carlo gradient descent (EM-MCGD), which estimates the gradients in Equations~\ref{eq:cross-entropy-gradients} by using the sample estimator $\tilde{\mprms}_{XZ} = \sum_{j=1}^m \V s_{XZ}(X^{(j)}, Z^{(j)})$ of $\mprms_{XZ}$ with $m$ exact model samples $(X^{(j)}, Z^{(j)}) \sim q_{XZ}$, and the sample estimator $\tilde{\mprms}_{XZ}^{(i)} = \sum_{j=1}^l \V s_{XZ}(X^{(i)}, Z^{(j)})$ of $\mprms_{XZ}^{(i)}$ over $l$ conditional samples $Z^{(j)} \sim q_{Z \mid X=X^{(i)}}$ for every $X^{(i)}$.

\begin{figure}[t]
	\includegraphics{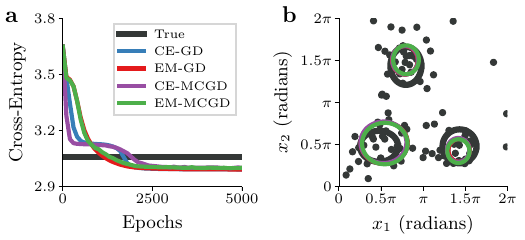}
	\caption{\emph{Training strategies for mixtures of von Mises densities.}\ \textbf{a:} True cross-entropy of a sample (black line) and model cross-entropies (coloured lines) over training epochs.\ \textbf{b:} Sample (black dots, $n = 100$) and precision ellipses (black lines) from the ground-truth von Mises mixture, and learned precision ellipses (coloured lines) for models trained with CE-GD (red), EM-GD (green), CE-MCGD (yellow), and EM-MCGD (blue) --- all ellipses heavily overlap.}\label{fig:learning-algorithms}
\end{figure}

We demonstrate the CE-GD, EM-GD, CE-MCGD, and EM-MCGD algorithms by training them on 100 sample points from a ground-truth mixture of 2D von Mises densities, where each component is a product of two von Mises distributions (Fig.~\ref{fig:learning-algorithms}). For CE-GD and EM-GD we compute the average gradient over the entire sample, so that each gradient step corresponds to one epoch. For CE-MCGD and EM-MCGD, we generated $m=10$ model samples, and $l=1$ conditional samples per training point $X^{(i)}$, and we use a batch size of 10 so that 10 gradient steps corresponds to one epoch. Finally, for the EM-based algorithms, we hold the conditional sufficient statistics $\tilde{\mprms}^{(i)}_{XZ}$ fixed for 100 epochs before re-estimating them. We find that all algorithms can perform well (Fig.~\ref{fig:learning-algorithms}\textbf{a}), and the learned densities for each model converge to a good approximation of the ground-truth density~(Fig.~\ref{fig:learning-algorithms}\textbf{a}).

\subsection{Linear Subspaces}

We highlight one final algorithmic trick for training a harmonium while restricting the solution to a linear subspace. Consider the harmonium $\mathcal H_{XZ}$ defined by the exponential families $\mathcal M_X$ and $\mathcal M_Z$, and suppose $\mathcal H'_{XZ}$ is the harmonium defined by the sufficient statistic $\V s'_{XZ} = \V A \cdot \V s_{XZ}$ for some matrix $\V A$ --- to ensure that $\mathcal H_{XZ}$ is minimal, $\V A$ should have full rank and at least as many columns as rows. In this case if $q'_{XZ} \in \mathcal H'_{XZ}$ with parameters $\eprms'_{XZ}$, then $q'_{XZ} \in \mathcal H_{XZ}$ with parameters $\eprms'_{XZ} \cdot \V A$. We can thus compute the mean parameters $\mprms'_{XZ}$ of $q'_{XZ}$ by first evaluating (or estimating) $\mprms_{XZ} = \V \tau(\eprms'_{XZ} \cdot \V A)$, and then using the linearity of expectations to compute $\mprms'_{XZ} = \V \tau'_{XZ}(\eprms'_{XZ}) = \V A \cdot \mprms_{XZ}$, where $\V \tau_{XZ}$ and $\V \tau'_{XZ}$ are the forward mappings of $\mathcal H_{XZ}$ and $\mathcal H'_{XZ}$, respectively. Consequently, for any of our proposed algorithms except EM, if we can use it to train $\mathcal H_{XZ}$, we can use it to train $\mathcal H'_{XZ}$. This technique is of course not relevant when we can compute the expectations in the subspace directly --- for example, FA is often applied to high-dimensional observables, and clever linear algebra allows one to avoid computing the sample covariances directly in the (high-dimensional) observable space that are required to implement EM\@. On the other hand, we could use this technique to train a Gaussian-Boltzmann distribution (e.g.\ Fig.~\ref{fig:gaussian-boltzmann}) while restricting the interactions between neurons in the latent Boltzmann machine to e.g.\ neighbouring neurons in a lattice.

\begin{figure}[t]
	\includegraphics{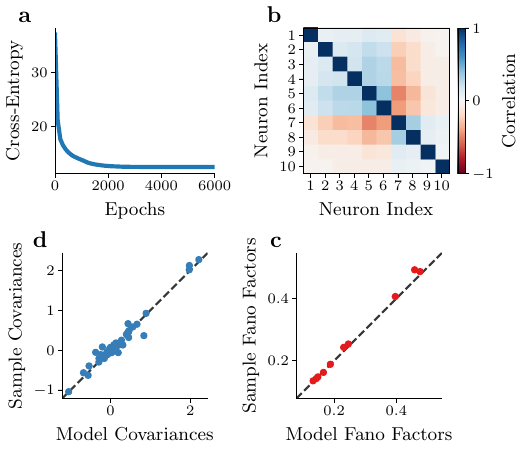}
	\caption{\emph{A factor analysis analogue for count-data.} \textbf{a:} Cross-entropy descent of a CoM-based mixture model trained on synthetic multivariate count data.\ \textbf{b:} Learned correlation matrix between the counts.\ \textbf{c-d:} Covariances (\textbf{c}) and Fano Factors (\textbf{d}) learned by the model compared to the sample estimators.}\label{fig:com-based-mixture}
\end{figure}

Another example for this technique is the so called CoM-based mixture model for modelling multivariate count-data~\cite{sokoloski_modelling_2021}. A Conway-Maxwell (CoM) Poisson distribution is an exponential family count distribution with distinct location and shape parameters~\cite{shmueli_useful_2005,stevenson_flexible_2016}, and CoM-based mixture models mix products of CoM-Poisson distributions. The parameters of the CoM-based mixture model are restricted (in a manner similar to FA) so that the latent category $K$ only modulates the location, and not the shape of the CoM-Poisson distributions. There is no efficient method for computing the expectations on this restricted parameter space, and so they are instead computed in the unrestricted parameter space.

Here we train a CoM-based mixture model on synthetic multivariate count data using the EM-GD algorithm (Fig.~\ref{fig:com-based-mixture}\textbf{a}), and show that it can capture rich correlation structure (Fig.~\ref{fig:com-based-mixture}\textbf{b, c}). The Fano factor is the variance over the mean of a random variable, and in contrast with mixtures of Poisson distributions, mixtures of CoM-Poisson distributions can model counts with Fano factors smaller than 1 (Fig.~\ref{fig:com-based-mixture}\textbf{d}).

\section{Discussion}

On one hand, Bayesian inference provides a general and rigorous approach to inferring information about unknown quantities given observations~\cite{jaynes_probability_2003}. On the other hand, learning a model through the method of maximum likelihood is grounded in an axiomatic characterization of how to minimize the divergence of a model from a target (i.e.\ empirical) distribution~\cite{shore_axiomatic_1980}. In this paper we developed a theory of exponential family LVMs we call conjugated harmoniums, that afford both exact Bayesian inference and direct maximum-likelihood learning. Moreover, we extended classical results on conjugate priors, and derived the most general conditions under which an arbitrary exponential family-distributed likelihood function can have a conjugate prior family.

Although our goal has been to avoid approximation schemes, we see this work as complementary to methods for approximate inference and learning. Variational autoencoders, for example, typically approximate the intractable posterior of the model with a tractable exponential family~\cite{shekhovtsov_vae_2021}, and this exponential family need not be Gaussian or categorical, as is often assumed~\cite{vahdat_undirected_2020}. Our theory thus provides numerous candidate models with sufficient tractability to model the latent space of a variational autoencoder.

LVMs also play a large role in computational neuroscience, where they are used to model how behavioural and environmental variables are encoded in low-dimensional submanifolds of neural activity~\cite{cunningham_dimensionality_2014,schneidman_towards_2016,langdon_unifying_2023}. In particular, Boltzmann machines and products of Poisson distributions are both examples of exponential families that are widely applied to modelling neural activity, and we showed how they both serve as effective components of conjugated harmoniums. Our theory also helps address a fundamental question in computational neuroscience: do neural circuits implement approximate inference and learning, or is neural connectivity constrained so that exact inference and learning are possible~\cite{shivkumar_probabilistic_2018,lange_bayesian_2023}? By expanding the scope of neural models that exhibit exact inference and learning, we should be better able to isolate those neural circuits that can indeed avoid approximation schemes.

As mentioned in Section~\ref{sec:harmoniums}, there are extensions of exponential family harmoniums to modelling multiple random variables with a graphical model structure~\cite{yang_graphical_2015,tansey_vector-space_2015}. While our theory of conjugation can be applied directly to these models by essentially coarse-graining the graph, in future work we aim to take full advantage of this graphical model structure. Indeed, sequential data models such as hidden Markov models and linear state space models both afford exact inference and learning~\cite{sarkka_bayesian_2013}, suggesting that a recursive strategy to conjugation over complex graphical models should be possible. Ultimately, a unified theory of conjugation across dynamical and graphical models would fundamentally push the boundaries of models that support exact inference and learning, opening doors to the efficient modeling of large numbers of observable and latent random variables with intricate, dynamic interdependencies.

\begin{appendix}
	\section{Parameter Transformations}\label{app:parameter-transformations}

	In this appendix we demonstrate how to derive the harmonium parameterization of a number of LVMs.

	\subsection{Mixture Models}\label{app:mixture-model-to-harmonium}

	Consider the mixture density $q_X(x ; \V w, \V \lambda_0, \ldots, \V \lambda_{d_K}) = \sum_{i=0}^d w_i q(x; \V \lambda_i)$ defined by the ${d_K}+1$ component densities $q(x; \V \lambda_i)$ with parameters $\V \lambda_i$, and weights that satisfy $w_0 = \sum_{i=1}^{d_K} w_i$ and $w_i > 0$. We can interpret this sum as the observable distribution $q_X$ of the LVM $q_{XK} = q_{X \mid K} \cdot q_X$ by introducing a latent index $K$ with values $\mathcal K = \{0, \dots, d_K\}$, and defining the likelihood and prior of the LVM as $q_{X \mid K}(x \mid k) = q(x; \V \lambda_k)$ and $q_K(k) = w_k$ for every $k$, respectively.

	Based on the formulas for the conjugation parameters in Theorem~\ref{thm:categorical-conjugation}, let us derive the natural parameters $\eprms_X$, $\eprms_K$, and $\iprms_{XK}$ given the component parameters $\V \lambda_0, \dots, \V \lambda_{d_K}$ and weights $\V w$ of $q_{XK}$. Let us assume that $\V f$ maps the component parameters $\V \lambda_i$ to the natural parameters of $q_{X \mid K = i} \in \mathcal M_X$. Because $\V s_K(0) = \V 0$, the definition of the harmonium likelihood (Eq.~\ref{eq:likelihood-density}) implies that $\eprms_X = \V f(\V \lambda_0)$. For $k > 0$, the natural parameters of $q_{X \mid Z=k}$ are $\eprms_X - \eprms_{XK,i}$, where $\eprms_{XK,i}$ is the $i$th column of $\iprms_{XZ}$, so that $\eprms_{XK,i} = \V f(\lambda_i) - \V f(\lambda_0)$. Finally, we can in general express a categorical density with weights $\V w$ in exponential family form with natural parameters $\eprms^*_K$ by defining $\theta^*_{K,i} = \log w_i - \log \sum_{j=1}^{d_K} w_j$ for $i \in \mathcal K$s. Therefore, according to Corollary~\ref{thm:conjugated-prior}, the latent biases $\eprms_K$ of $q_{XZ}$ are given by $\eprms^*_K - \rprms$.

	\subsection{Linear Gaussian Models}\label{app:linear-gaussian-to-harmonium}

	Where $\Normal(\V m, \V \Sigma)$ is a multivariate normal density with mean $\V m$ and covariance $\V \Sigma$, suppose we are given $q_{XZ}$ in terms of the prior $q_Z = \Normal(\V m_Z, \V \Sigma_Z)$ and likelihood $q_{X \mid Z = \V z} = \Normal(\V m_X + \V W \cdot \V z, \V \Sigma_X)$. We can express $q_{XZ} \in \mathcal G_{XZ}$ in the (homoscedastic) harmonium form
	\begin{equation*}
		q_{XZ}(\V x, \V z) \propto e^{\V x \cdot \eprms^m_X + \V x \cdot \iprms^\sigma_X \cdot \V x + \V z \cdot \eprms^m_Z + \V z \cdot \iprms^\sigma_Z \cdot \V z + \V x \cdot \iprms^{m}_{XZ} \cdot \V z},
	\end{equation*}
	where
	\begin{align*}
		\eprms^m_X      & = \V \Sigma_X^{-1} \cdot \V m_X,               \\
		\iprms^\sigma_X & = -\frac{1}{2} \V \Sigma_X^{-1},               \\
		\iprms^{m}_{XZ} & = \V \Sigma_X^{-1} \cdot \V W,                 \\
		\eprms^m_Z      & = \V \Sigma_Z^{-1} \cdot \V m_Z - \rprms^m,    \\
		\iprms^\sigma_Z & = -\frac{1}{2} \V \Sigma_Z^{-1} - \V P^\sigma,
	\end{align*}
	and $\rprms^m$ and $\V P^\sigma$ are given by Equations~\ref{eq:linear-gaussian-conjugation}.

	\section{Proofs}

	\printProofs{}

\end{appendix}

\begin{acks}[Acknowledgments]
	This paper has been shaped by countless conversations I have had over the years with colleagues and friends. Nevertheless, I would like to thank Philipp Berens in particular for additional conversations and support that allowed me to finalize this manuscript.
\end{acks}
\begin{funding}
	This research was supported by the Hertie Foundation (Gemeinnützige Hertie-Stiftung); the Deutsche Forschungsgesellschaft (DFG) Sonderforschungsbereich (SFB) 1233, ``Robust Vision: Inference Principles and Neural Mechanisms'', Teilprojekt (TP) 13, project number: 276693517; and the DFG Cluster of Excellence ``Machine Learning --- New Perspectives for Science'', EXC 2064.
\end{funding}

\begin{supplement}
	\stitle{Code for Numerical Libraries}
	\sdescription{This supplement contains the source code used for the numerical demonstrations in the manuscript. The code provides implementations of the inference and learning algorithms described in the paper. Numerous example scripts are also provided, a subset of which constitutes the implementations behind the demonstrations and plots featured in this paper. A publicly accessible version is also available at \url{https://github.com/alex404/goal}.}
\end{supplement}

\bibliographystyle{imsart-number} 
\bibliography{library}       

\begin{thebibliography}{50}

	\bibitem{ackley_learning_1985}
	\begin{barticle}[author]
		\bauthor{\bsnm{Ackley},~\bfnm{David~H.}\binits{D.~H.}},
		\bauthor{\bsnm{Hinton},~\bfnm{Geoffrey~E.}\binits{G.~E.}} \AND
		\bauthor{\bsnm{Sejnowski},~\bfnm{Terrence~J.}\binits{T.~J.}}
		(\byear{1985}).
		\btitle{A learning algorithm for {Boltzmann} machines}.
		\bjournal{Cognitive science}
		\bvolume{9}
		\bpages{147--169}.
	\end{barticle}
	\endbibitem

	\bibitem{amari_methods_2007}
	\begin{bbook}[author]
		\bauthor{\bsnm{Amari},~\bfnm{Shun-ichi}\binits{S.-i.}} \AND
		\bauthor{\bsnm{Nagaoka},~\bfnm{Hiroshi}\binits{H.}}
		(\byear{2007}).
		\btitle{Methods of information geometry}
		\bvolume{191}.
		\bpublisher{American Mathematical Soc.}
	\end{bbook}
	\endbibitem

	\bibitem{arnold_conjugate_1993}
	\begin{barticle}[author]
		\bauthor{\bsnm{Arnold},~\bfnm{Barry~C.}\binits{B.~C.}},
		\bauthor{\bsnm{Castillo},~\bfnm{Enrique}\binits{E.}} \AND
		\bauthor{\bsnm{Sarabia},~\bfnm{Jose~María}\binits{J.~M.}}
		(\byear{1993}).
		\btitle{Conjugate {Exponential} {Family} {Priors} {For} {Exponential} {Family}
				{Likelihoods}}.
		\bjournal{Statistics}
		\bvolume{25}
		\bpages{71--77}.
		\bnote{Publisher: Taylor \& Francis \_eprint:
			https://doi.org/10.1080/02331889308802432}.
		\bdoi{10.1080/02331889308802432}
	\end{barticle}
	\endbibitem

	\bibitem{arnold_conditionally_2001}
	\begin{barticle}[author]
		\bauthor{\bsnm{Arnold},~\bfnm{Barry~C.}\binits{B.~C.}},
		\bauthor{\bsnm{Castillo},~\bfnm{Enrique}\binits{E.}} \AND
		\bauthor{\bsnm{Sarabia},~\bfnm{José~María}\binits{J.~M.}}
		(\byear{2001}).
		\btitle{Conditionally {Specified} {Distributions}: {An} {Introduction} (with
		comments and a rejoinder by the authors)}.
		\bjournal{Statistical Science}
		\bvolume{16}
		\bpages{249--274}.
		\bnote{Publisher: Institute of Mathematical Statistics}.
		\bdoi{10.1214/ss/1009213728}
	\end{barticle}
	\endbibitem

	\bibitem{arnold_compatible_1989}
	\begin{barticle}[author]
		\bauthor{\bsnm{Arnold},~\bfnm{Barry~C.}\binits{B.~C.}} \AND
		\bauthor{\bsnm{Press},~\bfnm{S.~James}\binits{S.~J.}}
		(\byear{1989}).
		\btitle{Compatible conditional distributions}.
		\bjournal{Journal of the American Statistical Association}
		\bvolume{84}
		\bpages{152--156}.
	\end{barticle}
	\endbibitem

	\bibitem{beck_marginalization_2011}
	\begin{barticle}[author]
		\bauthor{\bsnm{Beck},~\bfnm{Jeff}\binits{J.}},
		\bauthor{\bsnm{Latham},~\bfnm{Peter}\binits{P.}} \AND
		\bauthor{\bsnm{Pouget},~\bfnm{Alexandre}\binits{A.}}
		(\byear{2011}).
		\btitle{Marginalization in {Neural} {Circuits} with {Divisive}
				{Normalization}}.
		\bjournal{The Journal of Neuroscience}
		\bvolume{31}
		\bpages{15310--15319}.
	\end{barticle}
	\endbibitem

	\bibitem{beck_complex_2012}
	\begin{bincollection}[author]
		\bauthor{\bsnm{Beck},~\bfnm{Jeff}\binits{J.}},
		\bauthor{\bsnm{Pouget},~\bfnm{Alexandre}\binits{A.}} \AND
		\bauthor{\bsnm{Heller},~\bfnm{Katherine~A}\binits{K.~A.}}
		(\byear{2012}).
		\btitle{Complex {Inference} in {Neural} {Circuits} with {Probabilistic}
				{Population} {Codes} and {Topic} {Models}}.
		In \bbooktitle{Advances in {Neural} {Information} {Processing} {Systems} 25}
		(\beditor{\bfnm{F.}\binits{F.}~\bsnm{Pereira}},
		\beditor{\bfnm{C.~J.~C.}\binits{C.~J.~C.}~\bsnm{Burges}},
		\beditor{\bfnm{L.}\binits{L.}~\bsnm{Bottou}} \AND
		\beditor{\bfnm{K.~Q.}\binits{K.~Q.}~\bsnm{Weinberger}}, eds.)
		\bpages{3059--3067}.
		\bpublisher{Curran Associates, Inc.}
	\end{bincollection}
	\endbibitem

	\bibitem{bengio_justifying_2009}
	\begin{barticle}[author]
		\bauthor{\bsnm{Bengio},~\bfnm{Yoshua}\binits{Y.}} \AND
		\bauthor{\bsnm{Delalleau},~\bfnm{Olivier}\binits{O.}}
		(\byear{2009}).
		\btitle{Justifying and generalizing contrastive divergence}.
		\bjournal{Neural Computation}
		\bvolume{21}
		\bpages{1601--1621}.
	\end{barticle}
	\endbibitem

	\bibitem{besag_spatial_1974}
	\begin{barticle}[author]
		\bauthor{\bsnm{Besag},~\bfnm{Julian}\binits{J.}}
		(\byear{1974}).
		\btitle{Spatial {Interaction} and the {Statistical} {Analysis} of {Lattice}
				{Systems}}.
		\bjournal{Journal of the Royal Statistical Society. Series B (Methodological)}
		\bvolume{36}
		\bpages{192--236}.
	\end{barticle}
	\endbibitem

	\bibitem{bishop_pattern_2006}
	\begin{bbook}[author]
		\bauthor{\bsnm{Bishop},~\bfnm{Christopher~M.}\binits{C.~M.}}
		(\byear{2006}).
		\btitle{Pattern recognition and machine learning}.
		\bseries{Information science and statistics}.
		\bpublisher{Springer}, \baddress{New York}.
	\end{bbook}
	\endbibitem

	\bibitem{boulanger-lewandowski_modeling_2012}
	\begin{binproceedings}[author]
		\bauthor{\bsnm{Boulanger-Lewandowski},~\bfnm{Nicolas}\binits{N.}},
		\bauthor{\bsnm{Bengio},~\bfnm{Yoshua}\binits{Y.}} \AND
		\bauthor{\bsnm{Vincent},~\bfnm{Pascal}\binits{P.}}
		(\byear{2012}).
		\btitle{Modeling {Temporal} {Dependencies} in {High}-dimensional {Sequences}:
		{Application} to {Polyphonic} {Music} {Generation} and {Transcription}}.
		In \bbooktitle{Proceedings of the 29th {International} {Coference} on
				{International} {Conference} on {Machine} {Learning}}.
		\bseries{{ICML}'12}
		\bpages{1881--1888}.
		\bpublisher{Omnipress}, \baddress{USA}.
	\end{binproceedings}
	\endbibitem

	\bibitem{cunningham_dimensionality_2014}
	\begin{barticle}[author]
		\bauthor{\bsnm{Cunningham},~\bfnm{John~P}\binits{J.~P.}} \AND
		\bauthor{\bsnm{Yu},~\bfnm{Byron~M}\binits{B.~M.}}
		(\byear{2014}).
		\btitle{Dimensionality reduction for large-scale neural recordings}.
		\bjournal{Nature Neuroscience}
		\bvolume{17}
		\bpages{1500--1509}.
		\bdoi{10.1038/nn.3776}
	\end{barticle}
	\endbibitem

	\bibitem{dempster_maximum_1977}
	\begin{barticle}[author]
		\bauthor{\bsnm{Dempster},~\bfnm{A.~P.}\binits{A.~P.}},
		\bauthor{\bsnm{Laird},~\bfnm{N.~M.}\binits{N.~M.}} \AND
		\bauthor{\bsnm{Rubin},~\bfnm{D.~B.}\binits{D.~B.}}
		(\byear{1977}).
		\btitle{Maximum {Likelihood} from {Incomplete} {Data} {Via} the {EM}
				{Algorithm}}.
		\bjournal{Journal of the Royal Statistical Society: Series B (Methodological)}
		\bvolume{39}
		\bpages{1--22}.
		\bnote{\_eprint:
			https://onlinelibrary.wiley.com/doi/pdf/10.1111/j.2517-6161.1977.tb01600.x}.
		\bdoi{10.1111/j.2517-6161.1977.tb01600.x}
	\end{barticle}
	\endbibitem

	\bibitem{diaconis_conjugate_1979}
	\begin{barticle}[author]
		\bauthor{\bsnm{Diaconis},~\bfnm{Persi}\binits{P.}} \AND
		\bauthor{\bsnm{Ylvisaker},~\bfnm{Donald}\binits{D.}}
		(\byear{1979}).
		\btitle{Conjugate {Priors} for {Exponential} {Families}}.
		\bjournal{The Annals of Statistics}
		\bvolume{7}
		\bpages{269--281}.
		\bnote{Publisher: Institute of Mathematical Statistics}.
	\end{barticle}
	\endbibitem

	\bibitem{durstewitz_state_2017}
	\begin{barticle}[author]
		\bauthor{\bsnm{Durstewitz},~\bfnm{Daniel}\binits{D.}}
		(\byear{2017}).
		\btitle{A state space approach for piecewise-linear recurrent neural networks
			for identifying computational dynamics from neural measurements}.
		\bjournal{PLOS Computational Biology}
		\bvolume{13}
		\bpages{e1005542}.
		\bnote{Publisher: Public Library of Science}.
		\bdoi{10.1371/journal.pcbi.1005542}
	\end{barticle}
	\endbibitem

	\bibitem{gerwinn_joint_2009}
	\begin{binproceedings}[author]
		\bauthor{\bsnm{Gerwinn},~\bfnm{Sebastian}\binits{S.}},
		\bauthor{\bsnm{Berens},~\bfnm{Philipp}\binits{P.}} \AND
		\bauthor{\bsnm{Bethge},~\bfnm{Matthias}\binits{M.}}
		(\byear{2009}).
		\btitle{A joint maximum-entropy model for binary neural population patterns and
			continuous signals}.
		In \bbooktitle{Advances in {Neural} {Information} {Processing} {Systems}}
		\bvolume{22}.
		\bpublisher{Curran Associates, Inc.}
	\end{binproceedings}
	\endbibitem

	\bibitem{hinton_training_2002}
	\begin{barticle}[author]
		\bauthor{\bsnm{Hinton},~\bfnm{Geoffrey~E.}\binits{G.~E.}}
		(\byear{2002}).
		\btitle{Training products of experts by minimizing contrastive divergence}.
		\bjournal{Neural computation}
		\bvolume{14}
		\bpages{1771--1800}.
	\end{barticle}
	\endbibitem

	\bibitem{hinton_fast_2006}
	\begin{barticle}[author]
		\bauthor{\bsnm{Hinton},~\bfnm{G.~E}\binits{G.~E.}},
		\bauthor{\bsnm{Osindero},~\bfnm{S.}\binits{S.}} \AND
		\bauthor{\bsnm{Teh},~\bfnm{Y.~W}\binits{Y.~W.}}
		(\byear{2006}).
		\btitle{A fast learning algorithm for deep belief nets}.
		\bjournal{Neural computation}
		\bvolume{18}
		\bpages{1527--1554}.
	\end{barticle}
	\endbibitem

	\bibitem{ho_denoising_2020}
	\begin{binproceedings}[author]
		\bauthor{\bsnm{Ho},~\bfnm{Jonathan}\binits{J.}},
		\bauthor{\bsnm{Jain},~\bfnm{Ajay}\binits{A.}} \AND
		\bauthor{\bsnm{Abbeel},~\bfnm{Pieter}\binits{P.}}
		(\byear{2020}).
		\btitle{Denoising {Diffusion} {Probabilistic} {Models}}.
		In \bbooktitle{Advances in {Neural} {Information} {Processing} {Systems}}
		\bvolume{33}
		\bpages{6840--6851}.
		\bpublisher{Curran Associates, Inc.}
	\end{binproceedings}
	\endbibitem

	\bibitem{jaynes_probability_2003}
	\begin{bbook}[author]
		\bauthor{\bsnm{Jaynes},~\bfnm{Edwin~T.}\binits{E.~T.}}
		(\byear{2003}).
		\btitle{Probability theory: the logic of science}.
		\bpublisher{Cambridge university press}.
	\end{bbook}
	\endbibitem

	\bibitem{kingma_adam_2014}
	\begin{barticle}[author]
		\bauthor{\bsnm{Kingma},~\bfnm{Diederik}\binits{D.}} \AND
		\bauthor{\bsnm{Ba},~\bfnm{Jimmy}\binits{J.}}
		(\byear{2014}).
		\btitle{Adam: {A} method for stochastic optimization}.
		\bjournal{arXiv preprint arXiv:1412.6980}.
	\end{barticle}
	\endbibitem

	\bibitem{kingma_introduction_2019}
	\begin{barticle}[author]
		\bauthor{\bsnm{Kingma},~\bfnm{Diederik~P.}\binits{D.~P.}} \AND
		\bauthor{\bsnm{Welling},~\bfnm{Max}\binits{M.}}
		(\byear{2019}).
		\btitle{An {Introduction} to {Variational} {Autoencoders}}.
		\bjournal{Foundations and Trends® in Machine Learning}
		\bvolume{12}
		\bpages{307--392}.
		\bdoi{10.1561/2200000056}
	\end{barticle}
	\endbibitem

	\bibitem{krefl_riemann-theta_2020}
	\begin{barticle}[author]
		\bauthor{\bsnm{Krefl},~\bfnm{Daniel}\binits{D.}},
		\bauthor{\bsnm{Carrazza},~\bfnm{Stefano}\binits{S.}},
		\bauthor{\bsnm{Haghighat},~\bfnm{Babak}\binits{B.}} \AND
		\bauthor{\bsnm{Kahlen},~\bfnm{Jens}\binits{J.}}
		(\byear{2020}).
		\btitle{Riemann-{Theta} {Boltzmann} machine}.
		\bjournal{Neurocomputing}
		\bvolume{388}
		\bpages{334--345}.
		\bdoi{10.1016/j.neucom.2020.01.011}
	\end{barticle}
	\endbibitem

	\bibitem{lake_humanlevel_2015}
	\begin{barticle}[author]
		\bauthor{\bsnm{Lake},~\bfnm{B.~M.}\binits{B.~M.}},
		\bauthor{\bsnm{Salakhutdinov},~\bfnm{R.}\binits{R.}} \AND
		\bauthor{\bsnm{Tenenbaum},~\bfnm{J.~B.}\binits{J.~B.}}
		(\byear{2015}).
		\btitle{Human-level concept learning through probabilistic program induction}.
		\bjournal{Science}
		\bvolume{350}
		\bpages{1332--1338}.
		\bdoi{10.1126/science.aab3050}
	\end{barticle}
	\endbibitem

	\bibitem{langdon_unifying_2023}
	\begin{barticle}[author]
		\bauthor{\bsnm{Langdon},~\bfnm{Christopher}\binits{C.}},
		\bauthor{\bsnm{Genkin},~\bfnm{Mikhail}\binits{M.}} \AND
		\bauthor{\bsnm{Engel},~\bfnm{Tatiana~A.}\binits{T.~A.}}
		(\byear{2023}).
		\btitle{A unifying perspective on neural manifolds and circuits for cognition}.
		\bjournal{Nature Reviews Neuroscience}
		\bvolume{24}
		\bpages{363--377}.
		\bnote{Publisher: Nature Publishing Group}.
		\bdoi{10.1038/s41583-023-00693-x}
	\end{barticle}
	\endbibitem

	\bibitem{lange_bayesian_2023}
	\begin{barticle}[author]
		\bauthor{\bsnm{Lange},~\bfnm{Richard~D.}\binits{R.~D.}},
		\bauthor{\bsnm{Shivkumar},~\bfnm{Sabyasachi}\binits{S.}},
		\bauthor{\bsnm{Chattoraj},~\bfnm{Ankani}\binits{A.}} \AND
		\bauthor{\bsnm{Haefner},~\bfnm{Ralf~M.}\binits{R.~M.}}
		(\byear{2023}).
		\btitle{Bayesian encoding and decoding as distinct perspectives on neural
			coding}.
		\bjournal{Nature Neuroscience}
		\bvolume{26}
		\bpages{2063--2072}.
		\bnote{Publisher: Nature Publishing Group}.
		\bdoi{10.1038/s41593-023-01458-6}
	\end{barticle}
	\endbibitem

	\bibitem{ma_bayesian_2006}
	\begin{barticle}[author]
		\bauthor{\bsnm{Ma},~\bfnm{Wei~Ji}\binits{W.~J.}},
		\bauthor{\bsnm{Beck},~\bfnm{Jeff}\binits{J.}},
		\bauthor{\bsnm{Latham},~\bfnm{Peter}\binits{P.}} \AND
		\bauthor{\bsnm{Pouget},~\bfnm{Alexandre}\binits{A.}}
		(\byear{2006}).
		\btitle{Bayesian inference with probabilistic population codes}.
		\bjournal{Nature Neuroscience}
		\bvolume{9}
		\bpages{1432--1438}.
		\bdoi{10.1038/nn1790}
	\end{barticle}
	\endbibitem

	\bibitem{martin_approximating_2024}
	\begin{barticle}[author]
		\bauthor{\bsnm{Martin},~\bfnm{Gael~M.}\binits{G.~M.}},
		\bauthor{\bsnm{Frazier},~\bfnm{David~T.}\binits{D.~T.}} \AND
		\bauthor{\bsnm{Robert},~\bfnm{Christian~P.}\binits{C.~P.}}
		(\byear{2024}).
		\btitle{Approximating {Bayes} in the 21st {Century}}.
		\bjournal{Statistical Science}
		\bvolume{39}
		\bpages{20--45}.
		\bnote{Publisher: Institute of Mathematical Statistics}.
		\bdoi{10.1214/22-STS875}
	\end{barticle}
	\endbibitem

	\bibitem{murphy_probabilistic_2023}
	\begin{bbook}[author]
		\bauthor{\bsnm{Murphy},~\bfnm{Kevin~P.}\binits{K.~P.}}
		(\byear{2023}).
		\btitle{Probabilistic machine learning: {Advanced} topics}.
		\bpublisher{MIT press}.
	\end{bbook}
	\endbibitem

	\bibitem{roweis_unifying_1999}
	\begin{barticle}[author]
		\bauthor{\bsnm{Roweis},~\bfnm{Sam}\binits{S.}} \AND
		\bauthor{\bsnm{Ghahramani},~\bfnm{Zoubin}\binits{Z.}}
		(\byear{1999}).
		\btitle{A {Unifying} {Review} of {Linear} {Gaussian} {Models}}.
		\bjournal{Neural Computation}
		\bvolume{11}
		\bpages{305--345}.
		\bdoi{10.1162/089976699300016674}
	\end{barticle}
	\endbibitem

	\bibitem{salakhutdinov_efficient_2012}
	\begin{barticle}[author]
		\bauthor{\bsnm{Salakhutdinov},~\bfnm{Ruslan}\binits{R.}} \AND
		\bauthor{\bsnm{Hinton},~\bfnm{Geoffrey}\binits{G.}}
		(\byear{2012}).
		\btitle{An {Efficient} {Learning} {Procedure} for {Deep} {Boltzmann}
				{Machines}}.
		\bjournal{Neural Computation}
		\bvolume{24}
		\bpages{1967--2006}.
		\bdoi{10.1162/NECO_a_00311}
	\end{barticle}
	\endbibitem

	\bibitem{salakhutdinov_learning_2013}
	\begin{barticle}[author]
		\bauthor{\bsnm{Salakhutdinov},~\bfnm{Ruslan}\binits{R.}},
		\bauthor{\bsnm{Tenenbaum},~\bfnm{Joshua~B.}\binits{J.~B.}} \AND
		\bauthor{\bsnm{Torralba},~\bfnm{Antonio}\binits{A.}}
		(\byear{2013}).
		\btitle{Learning with {Hierarchical}-{Deep} {Models}}.
		\bjournal{IEEE Transactions on Pattern Analysis and Machine Intelligence}
		\bvolume{35}
		\bpages{1958--1971}.
		\bnote{Conference Name: IEEE Transactions on Pattern Analysis and Machine
			Intelligence}.
		\bdoi{10.1109/TPAMI.2012.269}
	\end{barticle}
	\endbibitem

	\bibitem{schneidman_towards_2016}
	\begin{barticle}[author]
		\bauthor{\bsnm{Schneidman},~\bfnm{Elad}\binits{E.}}
		(\byear{2016}).
		\btitle{Towards the design principles of neural population codes}.
		\bjournal{Current Opinion in Neurobiology}
		\bvolume{37}
		\bpages{133--140}.
		\bdoi{10.1016/j.conb.2016.03.001}
	\end{barticle}
	\endbibitem

	\bibitem{shekhovtsov_vae_2021}
	\begin{binproceedings}[author]
		\bauthor{\bsnm{Shekhovtsov},~\bfnm{Alexander}\binits{A.}},
		\bauthor{\bsnm{Schlesinger},~\bfnm{Dmitrij}\binits{D.}} \AND
		\bauthor{\bsnm{Flach},~\bfnm{Boris}\binits{B.}}
		(\byear{2021}).
		\btitle{{VAE} {Approximation} {Error}: {ELBO} and {Exponential} {Families}}.
	\end{binproceedings}
	\endbibitem

	\bibitem{shivkumar_probabilistic_2018}
	\begin{binproceedings}[author]
		\bauthor{\bsnm{Shivkumar},~\bfnm{Sabyasachi}\binits{S.}},
		\bauthor{\bsnm{Lange},~\bfnm{Richard}\binits{R.}},
		\bauthor{\bsnm{Chattoraj},~\bfnm{Ankani}\binits{A.}} \AND
		\bauthor{\bsnm{Haefner},~\bfnm{Ralf}\binits{R.}}
		(\byear{2018}).
		\btitle{A probabilistic population code based on neural samples}.
		In \bbooktitle{Advances in {Neural} {Information} {Processing} {Systems}}
		\bvolume{31}.
		\bpublisher{Curran Associates, Inc.}
	\end{binproceedings}
	\endbibitem

	\bibitem{shmueli_useful_2005}
	\begin{barticle}[author]
		\bauthor{\bsnm{Shmueli},~\bfnm{Galit}\binits{G.}},
		\bauthor{\bsnm{Minka},~\bfnm{Thomas~P.}\binits{T.~P.}},
		\bauthor{\bsnm{Kadane},~\bfnm{Joseph~B.}\binits{J.~B.}},
		\bauthor{\bsnm{Borle},~\bfnm{Sharad}\binits{S.}} \AND
		\bauthor{\bsnm{Boatwright},~\bfnm{Peter}\binits{P.}}
		(\byear{2005}).
		\btitle{A useful distribution for fitting discrete data: revival of the
			{Conway}–{Maxwell}–{Poisson} distribution}.
		\bjournal{Journal of the Royal Statistical Society: Series C (Applied
			Statistics)}
		\bvolume{54}
		\bpages{127--142}.
		\bdoi{10.1111/j.1467-9876.2005.00474.x}
	\end{barticle}
	\endbibitem

	\bibitem{shore_axiomatic_1980}
	\begin{barticle}[author]
		\bauthor{\bsnm{Shore},~\bfnm{J.}\binits{J.}} \AND
		\bauthor{\bsnm{Johnson},~\bfnm{R.~W.}\binits{R.~W.}}
		(\byear{1980}).
		\btitle{Axiomatic derivation of the principle of maximum entropy and the
			principle of minimum cross-entropy}.
		\bjournal{Information Theory, IEEE Transactions on}
		\bvolume{26}
		\bpages{26--37}.
	\end{barticle}
	\endbibitem

	\bibitem{smolensky_information_1986}
	\begin{btechreport}[author]
		\bauthor{\bsnm{Smolensky},~\bfnm{Paul}\binits{P.}}
		(\byear{1986}).
		\btitle{Information {Processing} in {Dynamical} {Systems}: {Foundations} of
			{Harmony} {Theory}}
		\btype{Technical Report},
		\bpublisher{COLORADO UNIV AT BOULDER DEPT OF COMPUTER SCIENCE}.
		\bnote{Section: Technical Reports}.
	\end{btechreport}
	\endbibitem

	\bibitem{sokoloski_modelling_2021}
	\begin{barticle}[author]
		\bauthor{\bsnm{Sokoloski},~\bfnm{Sacha}\binits{S.}},
		\bauthor{\bsnm{Aschner},~\bfnm{Amir}\binits{A.}} \AND
		\bauthor{\bsnm{Coen-Cagli},~\bfnm{Ruben}\binits{R.}}
		(\byear{2021}).
		\btitle{Modelling the neural code in large populations of correlated neurons}.
		\bjournal{eLife}
		\bvolume{10}
		\bpages{e64615}.
		\bnote{Publisher: eLife Sciences Publications, Ltd}.
		\bdoi{10.7554/eLife.64615}
	\end{barticle}
	\endbibitem

	\bibitem{supplement}
	\begin{barticle}[author]
		\bauthor{\bsnm{Sokoloski},~\bfnm{Sacha}\binits{S.}}
		(\byear{2025}).
		\btitle{Supplementary Material for ``A Unified Theory of Exact Inference and Learning in Exponential Family Latent Variable Models''}.
		\bjournal{Statistical Science Supplement}.
	\end{barticle}
	\endbibitem

	\bibitem{stevenson_flexible_2016}
	\begin{barticle}[author]
		\bauthor{\bsnm{Stevenson},~\bfnm{Ian~H.}\binits{I.~H.}}
		(\byear{2016}).
		\btitle{Flexible models for spike count data with both over- and under-
			dispersion}.
		\bjournal{Journal of Computational Neuroscience}
		\bvolume{41}
		\bpages{29--43}.
		\bdoi{10.1007/s10827-016-0603-y}
	\end{barticle}
	\endbibitem

	\bibitem{sutskever_convergence_2010}
	\begin{binproceedings}[author]
		\bauthor{\bsnm{Sutskever},~\bfnm{Ilya}\binits{I.}} \AND
		\bauthor{\bsnm{Tieleman},~\bfnm{Tijmen}\binits{T.}}
		(\byear{2010}).
		\btitle{On the convergence properties of contrastive divergence}.
		In \bbooktitle{International {Conference} on {Artificial} {Intelligence} and
				{Statistics}}
		\bpages{789--795}.
	\end{binproceedings}
	\endbibitem

	\bibitem{sarkka_bayesian_2013}
	\begin{bbook}[author]
		\bauthor{\bsnm{Särkkä},~\bfnm{Simo}\binits{S.}}
		(\byear{2013}).
		\btitle{Bayesian filtering and smoothing}
		\bvolume{3}.
		\bpublisher{Cambridge University Press}.
	\end{bbook}
	\endbibitem

	\bibitem{tansey_vector-space_2015}
	\begin{binproceedings}[author]
		\bauthor{\bsnm{Tansey},~\bfnm{Wesley}\binits{W.}},
		\bauthor{\bsnm{Padilla},~\bfnm{Oscar Hernan~Madrid}\binits{O.~H.~M.}},
		\bauthor{\bsnm{Suggala},~\bfnm{Arun~Sai}\binits{A.~S.}} \AND
		\bauthor{\bsnm{Ravikumar},~\bfnm{Pradeep}\binits{P.}}
		(\byear{2015}).
		\btitle{Vector-{Space} {Markov} {Random} {Fields} via {Exponential}
			{Families}}.
		In \bbooktitle{{PMLR}}
		\bpages{684--692}.
	\end{binproceedings}
	\endbibitem

	\bibitem{taylor_two_2011}
	\begin{barticle}[author]
		\bauthor{\bsnm{Taylor},~\bfnm{Graham~W.}\binits{G.~W.}},
		\bauthor{\bsnm{Hinton},~\bfnm{Geoffrey~E.}\binits{G.~E.}} \AND
		\bauthor{\bsnm{Roweis},~\bfnm{Sam~T.}\binits{S.~T.}}
		(\byear{2011}).
		\btitle{Two distributed-state models for generating high-dimensional time
			series}.
		\bjournal{Journal of Machine Learning Research}
		\bvolume{12}
		\bpages{1025--1068}.
	\end{barticle}
	\endbibitem

	\bibitem{vahdat_undirected_2020}
	\begin{binproceedings}[author]
		\bauthor{\bsnm{Vahdat},~\bfnm{Arash}\binits{A.}},
		\bauthor{\bsnm{Andriyash},~\bfnm{Evgeny}\binits{E.}} \AND
		\bauthor{\bsnm{Macready},~\bfnm{William}\binits{W.}}
		(\byear{2020}).
		\btitle{Undirected {Graphical} {Models} as {Approximate} {Posteriors}}.
		In \bbooktitle{Proceedings of the 37th {International} {Conference} on
				{Machine} {Learning}}
		\bpages{9680--9689}.
		\bpublisher{PMLR}
		\bnote{ISSN: 2640-3498}.
	\end{binproceedings}
	\endbibitem

	\bibitem{vertes_flexible_2018}
	\begin{bincollection}[author]
		\bauthor{\bsnm{Vértes},~\bfnm{Eszter}\binits{E.}} \AND
		\bauthor{\bsnm{Sahani},~\bfnm{Maneesh}\binits{M.}}
		(\byear{2018}).
		\btitle{Flexible and accurate inference and learning for deep generative
			models}.
		In \bbooktitle{Advances in {Neural} {Information} {Processing} {Systems} 31}
		(\beditor{\bfnm{S.}\binits{S.}~\bsnm{Bengio}},
		\beditor{\bfnm{H.}\binits{H.}~\bsnm{Wallach}},
		\beditor{\bfnm{H.}\binits{H.}~\bsnm{Larochelle}},
		\beditor{\bfnm{K.}\binits{K.}~\bsnm{Grauman}},
		\beditor{\bfnm{N.}\binits{N.}~\bsnm{Cesa-Bianchi}} \AND
		\beditor{\bfnm{R.}\binits{R.}~\bsnm{Garnett}}, eds.)
		\bpages{4166--4175}.
		\bpublisher{Curran Associates, Inc.}
	\end{bincollection}
	\endbibitem

	\bibitem{wainwright_graphical_2008}
	\begin{barticle}[author]
		\bauthor{\bsnm{Wainwright},~\bfnm{Martin~J.}\binits{M.~J.}} \AND
		\bauthor{\bsnm{Jordan},~\bfnm{Michael~I.}\binits{M.~I.}}
		(\byear{2008}).
		\btitle{Graphical models, exponential families, and variational inference}.
		\bjournal{Foundations and Trends® in Machine Learning}
		\bvolume{1}
		\bpages{1--305}.
	\end{barticle}
	\endbibitem

	\bibitem{welling_exponential_2005}
	\begin{bincollection}[author]
		\bauthor{\bsnm{Welling},~\bfnm{Max}\binits{M.}},
		\bauthor{\bsnm{Rosen-zvi},~\bfnm{Michal}\binits{M.}} \AND
		\bauthor{\bsnm{Hinton},~\bfnm{Geoffrey~E}\binits{G.~E.}}
		(\byear{2005}).
		\btitle{Exponential {Family} {Harmoniums} with an {Application} to
				{Information} {Retrieval}}.
		In \bbooktitle{Advances in {Neural} {Information} {Processing} {Systems} 17}
		(\beditor{\bfnm{L.~K.}\binits{L.~K.}~\bsnm{Saul}},
		\beditor{\bfnm{Y.}\binits{Y.}~\bsnm{Weiss}} \AND
		\beditor{\bfnm{L.}\binits{L.}~\bsnm{Bottou}}, eds.)
		\bpages{1481--1488}.
		\bpublisher{MIT Press}.
	\end{bincollection}
	\endbibitem

	\bibitem{yang_graphical_2012}
	\begin{binproceedings}[author]
		\bauthor{\bsnm{Yang},~\bfnm{Eunho}\binits{E.}},
		\bauthor{\bsnm{Allen},~\bfnm{Genevera}\binits{G.}},
		\bauthor{\bsnm{Liu},~\bfnm{Zhandong}\binits{Z.}} \AND
		\bauthor{\bsnm{Ravikumar},~\bfnm{Pradeep}\binits{P.}}
		(\byear{2012}).
		\btitle{Graphical {Models} via {Generalized} {Linear} {Models}}.
		In \bbooktitle{Advances in {Neural} {Information} {Processing} {Systems}}
		\bvolume{25}.
		\bpublisher{Curran Associates, Inc.}
	\end{binproceedings}
	\endbibitem

	\bibitem{yang_graphical_2015}
	\begin{barticle}[author]
		\bauthor{\bsnm{Yang},~\bfnm{Eunho}\binits{E.}},
		\bauthor{\bsnm{Ravikumar},~\bfnm{Pradeep}\binits{P.}},
		\bauthor{\bsnm{Allen},~\bfnm{Genevera~I.}\binits{G.~I.}} \AND
		\bauthor{\bsnm{Liu},~\bfnm{Zhandong}\binits{Z.}}
		(\byear{2015}).
		\btitle{Graphical models via univariate exponential family distributions}.
		\bjournal{Journal of Machine Learning Research}
		\bvolume{16}
		\bpages{3813--3847}.
	\end{barticle}
	\endbibitem

\end{thebibliography}


\end{document}